\documentclass[12pt]{article}
\usepackage{amsfonts}
\usepackage{amsmath}
\usepackage{amssymb}
\usepackage{bbm}
\usepackage[usenames,dvipsnames]{color}
\usepackage{float}
\usepackage{graphicx}
\usepackage{epstopdf}
\usepackage{lscape}
\usepackage[comma]{natbib}
\usepackage{ntheorem}
\usepackage{rotating}
\usepackage{tabularx}
\usepackage{times}
\usepackage{mathtools}
\usepackage{multirow}
\graphicspath{ {figures/} }
\usepackage{subcaption}
\usepackage{epstopdf}
\DeclareGraphicsExtensions{.eps,.png,.pdf}
\usepackage{accents}
\usepackage{commath}
\usepackage{bm}
\usepackage{booktabs}
\usepackage[utf8]{inputenc}

\usepackage[margin=1in]{geometry}

\newcommand{\qed}{\hfill $\blacksquare$}
\numberwithin{equation}{section}

\allowdisplaybreaks

\theoremstyle{plain}
\newtheorem{theorem}{Theorem}[section]
\newenvironment{proof}[1][Proof]{\begin{trivlist}
\item[\hskip \labelsep {\bfseries #1.}]}{\end{trivlist}}

\newtheorem{lemma}[theorem]{Lemma}
\newtheorem*{remark}{Remark}

\newtheorem{assumption}{Assumption}
\newtheorem{corollary}{Corollary}[theorem]

\DeclarePairedDelimiter\ceil{\lceil}{\rceil}

\DeclareMathOperator*{\argmax}{arg\,max}
\newcommand{\E}{\mathbb{E}}

\begin{document}

\title{Optimal sequential treatment allocation}
\author{Anders Bredahl Kock and Martin Thyrsgaard\footnote{Address for correspondence: anders.kock@economics.ox.ac.uk. This work was supported by CREATES which is funded by the Danish National Research Foundation (DNRF78).}\\
University of Oxford, Aarhus University and CREATES}
\date{\today}
\maketitle

\begin{abstract}

In treatment allocation problems the individuals to be treated often arrive sequentially. We study a problem in which the policy maker is not only interested in the expected cumulative welfare but is also concerned about the uncertainty/risk of the treatment outcomes. At the outset, the total number of treatment assignments to be made may even be unknown. A sequential treatment policy which attains the minimax optimal regret is proposed. We also demonstrate that the expected number of suboptimal treatments only grows slowly in the number of treatments. Finally, we study a setting where outcomes are only observed with delay.

\vspace{0.5cm}
  
\end{abstract}


\section{Introduction}
A policy maker must often assign treatments gradually as the individuals to be treated do not arrive simultaneously. For example, people become unemployed gradually throughout the year and assignment to one of several unemployment programs is often made shortly thereafter. Similarly, patients with too high blood pressure arrive gradually to a medical clinic and the doctor assigns one of several treatments to each of them. The policy maker or doctor gradually accrues information by observing the outcome of previous treatments prior to the next assignment. Throughout the paper we shall use these two examples as illustrations of our results and be particularly concerned with how treatments should be assigned in order to maximize welfare. In doing so one faces a tradeoff between exploring which treatment works best and exploiting the information gathered so far from previous assignments in order to assign the best treatment to as many individuals as possible. 

The above setup is in stark contrast to typical estimation of treatment effects where one presupposes the existence of a data set of a certain size $N$ (perhaps obtained from a randomized control trial). Thus, in the typical setting, the size and composition of the data set are determined prior to estimation. Based on this given data set, treatment effects are estimated and assignments are made. We consider the case where the observed data that the treatment assignments must be based on is a part of the policy in the sense that it depends on the previous choices of the the policy maker. Thus, the policy maker enters already in the design phase of the treatment program and can adjust the experiment as data accumulates. In other words, he decides \textit{how} to draw the sample and thus its composition by the allocations he makes. Furthermore, the sample size itself may be a random variable unknown to the policy maker as he does not know a priori how many individuals will become unemployed in the course of the year that the program is scheduled to run, and the exact shape of a good treatment rule will depend on the expected number of individuals to be treated: if many individuals are expected to become unemployed in the course of the year it might be beneficial to experiment relatively more in the beginning to harvest the benefits of increased information later on. 

We contribute by considering a setting where the desirability of a treatment cannot be measured only by its expected outcome. A sensible welfare function must take into account the risk of a treatment. For example, it may well be that drug A is expected to lower the blood pressure slightly more than drug B but A might still not be preferred if it is much more risky than B. In this paper we shall measure the \textit{risk} of a treatment by its variance and take into account that mean as well as variance may be relevant in determining the most desirable treatment. Thus, we push beyond the classic focus on first moments (expected treatment outcomes) and take into account that also second moments (uncertainty about treatment outcome) play an important role in determining the most desirable treatment. We also indicate how one may incorporate more than two moments into the welfare function thus allowing welfare functions that allow for policy makers to be, say, skewness averse. This underscores the central idea of the paper: to go beyond focusing solely on the location (expected value) of the outcome distribution of treatments but also to focus on the \textit{shape} as measured by higher moments. 
We study a treatment policy, which we call the \textit{sequential treatment policy}, and show that it achieves the minimax optimal regret compared to the infeasible policy that knows in advance which treatment is best for each individual and assigns this. An upper bound on the expected number of times that the sequential treatment policy assigns any suboptimal treatment is provided as well. This is an important ethical guarantee since it ensures that the minimax optimal regret is not obtained at the cost of wild experimentation or maltreatment of many individuals in order to achieve a greater cumulative welfare in the long run. 

In addition, we contribute by studying the properties of the sequential treatment policy when the outcomes of previous treatments are observed only with delay. In a medical trial, for example, one may choose to delay the measurement of the outcome of the treatment in order to obtain more precise information of the effect of a certain drug. As it takes time for the effect of a drug to set in the delaying of the measurement can lead to more precise information on the effect of the drug. The price of this delay is that less information is available when treating other patients prior to the measurement being made. Thus, there is a tradeoff between obtaining imprecise information quickly (by making the measurement shortly after the treatment) and obtaining more precise information later (by postponing the measurement). We quantify this tradeoff and indicate the optimal delay (when this is a choice variable) and establish that our policy is guaranteed to deliver high welfare even in this setting.

Furthermore, we allow for a setting where individuals, and thus information, may arrive in batches. For example, people do not get assigned to an unemployment program on the exact day they become unemployed as new programs might only start once a month. Thus, people are pooled and as a result data arrives in batches. This setup strikes a middle ground between the \textit{bandit} framework where individuals arrive one-by-one and the classic treatment effect framework where a data set of size $N$ is presupposed. In our setting we also allow $N$ to be an unknown random variable which is important as the length of the treatment period may not be known at the beginning of the treatment period. 

Our approach easily accommodates practical policy concerns restricting the type of treatment rules that are feasible. For instance, the policy maker may want rules that depend on the individual's characteristics in a simple way due to political or ethical reasons.

It should be noted that the goal of this paper is not to test whether one treatment is better than the other ones at the end of the treatment period. This would amount to a pure exploration problem where the sole purpose of the sampling is to maximize the amount of information at the end of the sample without regard to the welfare of the treated individuals. While this problem is interesting in its own right it is often not viable for ethical reasons in the social sciences. Instead, the problem under investigation here is how to sample (assign treatments) in order to maximize the expected cumulative welfare of the treated individuals. That being said, we also propose a policy, the \textit{out-of-sample} policy, which indicates how treat individuals \textit{after} running the sequential treatment policy in the initial period. This is done in the setting where the total number of treatments is known from the outset. The out-of-sample policy is guaranteed to yield high welfare thus indicating that a lot has been learned about the individual treatments in our welfare maximization problem even though maximizing information was not the objective. Heuristically, the reason for non-negligible learning from observing the outcomes of the sequential treatment policy even when its purpose is to maximize welfare is that it has to experiment with the available treatments to make sure that one is assigning the best one often. Thus, learning is build into the welfare maximization of the initial period.

\subsection{Related literature}
Our paper is related to two strands of literature: the literature on statistical treatment rules in econometrics and the one on bandit problems on sequential allocation. In the former \cite{Manski2004} proposed \textit{conditional empirical success} (CES) rules which take a finite partition of the covariate space and on each set of this partition dictate to assign the treatment with the highest sample average. When implementing CES rules one must decide on how fine to choose the partition of the covariate space and thus faces a tradeoff between using highly individualized rules and having enough data to accurately estimate the treatment effects for each group in the partition. Among other things, \cite{Manski2004} provides sufficient conditions for full individualization to be optimal. The tradeoff between full individualization of treatments and having sufficient data to estimate the treatment effects accurately is also found in our dynamic treatment setting.

\cite{Stoye2009} showed that if one does not restrict how outcomes vary with covariates then full individualization is alway minimax optimal. Thus, if age is a covariate, information on treatment effects for 30 year olds should not be used when making treatment decisions for 31 year olds. This result relies on the fact that without any restrictions on  how the outcome distribution varies with covariates, this relationship could be infinitely wiggly such that even similar individuals may carry no information about how treatments affect the other person. Also, as the support of the covariate vector grows, these ``no-cross-covariate" rules become no-data rules as for many values of the covariates there will be no observations. This is certainly the case for continuous covariates. Our assumptions rule out such wiggliness as no practical policy can be expected to work well in such a setting. 

Furthermore, our work is related to the recent paper by \cite{Kitagawa2015} who consider treatment allocation through an \textit{empirical welfare maximization} lens. The authors take the view that realistic policies are often constrained to be simple due to ethical, legislative, or political reasons. Using techniques from empirical risk minimization they show how their procedure is minimax optimal within the considered class of realistic policies. Our approach is related to theirs in that we also allow the policy maker to focus on simple rules in the dynamic framework. Furthermore, \cite{athey2017efficient} have used concepts from semiparametric efficiency theory to establish regret bounds that scale with the semiparametrically efficient variance.

Other papers on statistical treatment rules in econometrics focusing on the case where the sample is given include \cite{chamberlain2000econometrics, dehejia2005program, hirano2009asymptotics, bhattacharya2012inferring, stoye2012minimax, tetenov2012statistical} and \cite{kasy2014using}.

The most important distinguishing feature of our work compared to the classic literature on statistical treatment rules is that we are working in a sequential setting where the individuals to be treated arrive gradually. Thus, we do not have a data set of size $N$ at our disposal from the outset based on which the best treatment must be found. The sequential setting poses new challenges such as not maltreating too many individuals in the search for the best treatment and how to handle data that arrives in batches as well as treatment outcomes that only are observed with delay. We shall address all of these in this work. Consequently, our paper is also related to the vast literature on bandit problems. In the classic bandit problems one seeks to maximize the expected cumulative reward from pulling arms with unknown means one by one. In a seminal paper \cite{Robbins1952} introduced a class of bandit problems and proposed some initial solutions guaranteeing that the average reward will converge to the mean of the best arm.

Broadly speaking, bandit problems can be classified into three categories based on the nature of the reward process: i) stochastic bandits where the arms are iid across time, ii) the markovian setting where the state of the arms changes according to a Markov process, iii) the adversarial setting in which nature chooses (an adversarial) distribution of rewards at the same time as the experimenter pulls an arm. In this work we focus on the stochastic setting as patients to be treated or unemployed individuals to be assigned to job training programs do not generally coordinate their effort against the doctor or policy maker in an adversarial manner. In the medical example in particular, the interests of the doctor and patient are often well-aligned. Furthermore, the markovian setting is concerned with infinite time horizons amounting to infinitely many treatments being made. In this work we are interested in the case where we have to make a finite, albeit often unknown, number of treatments. That being said, we certainly believe that also the adversarial setting or the markovian setting can be of interest to study in the context of sequential treatment allocation problems. In the latter setting, the Gittins index, \cite{gittins1979bandit}, is the most famous procedure. 

The first paper which considered bandit problems where one observes a covariate prior to making an allocation decision was \cite{woodroofe1979one} who made a parametric assumption on how covariates affect outcomes. The first work allowing covariates to affect the distribution of outcomes in a nonparametric way was \cite{yang2002randomized}. For an excellent review of the literature on bandit problems we refer to \cite{bubeck2012regret} who also elaborate more on the three fundamental settings and give further references. From an algorithmic point of view our work is related to \cite{CovBandit} who introduced a successive elimination (SE) policy of suboptimal arms. Their policy is in turn related to the work of \cite{even2003action}. 

Compared to the existing literature on bandits we contribute on several fronts. Most notably, we introduce potentially non-linear welfare functions depending not only on the mean treatment outcome but also on the variance of the outcome. Allowing the uncertainty (variance) of the treatment outcome to enter the welfare function is important as a policy maker may not only target the treatment with the highest expected outcome. It is likely that he also takes into account how risky the treatment is. This is a non-trivial extension of the classic bandit (and treatment) setting which has focused only on means and allows us to capture the dynamic risk-return tradeoff facing a policy maker when deciding which treatment to assign. From a technical point of view this is challenging since one must control the finite sample estimation error of estimators of the first and second moments of the treatment outcome distribution as well as non-linear transformations thereof i order to provide finite sample performance guarantees of our treatment policy. In addition, we consider the case where the outcome of treatments is only observed with delay. To the best of our knowledge, the consequences of delay and how to optimally deal with it have not been studied yet. As explained delay creates a tradeoff between obtaining imprecise information quickly and obtaining precise information later. This makes the sequential treatment problem more challenging as the policy maker must now also choose when to make a measurement in addition to which treatment to assign. Third, we provide upper bounds on the regret of the sequential treatment  policy for any choice of grouping individuals. These upper bounds depend on the geometry of the chosen grouping. Allowing for groups of general shapes is important to inform policy makers about how exactly their choice of grouping individuals affects regret since the choice of groups achieving minimax regret may not always be politically or ethically feasible. Fourth, we provide upper bounds on the expected number of suboptimal treatments our policy assigns. Fifth, we quantify how much has been learned in the course of the treatment period by proposing the out-of-sample policy. Our regret bounds show that even though we seek to maximize the cumulative welfare by our treatment assignments enough is learned in order to guarantee a low regret out of sample. 

The multi-armed bandit setup has also been used in the context of social learning and strategic experimentation in the works of e.g. \cite{bolton1999strategic}, \cite{keller2005strategic} and \cite{klein2011negatively}. Here several agents have to choose the amount of experimentation (pulling a risky arm) taking into account that the information obtained will be available to all other players as well. While the agents have an incentive to free ride they also want to experiment in order to bring forward the time where extra information is generated.

The term \textit{optimal sequential treatment allocation} in the bandit framework as discussed in this paper should not be confused with similar terms in the medical statistics literature. In that literature \textit{adaptive treatment strategies/adaptive interventions} and \textit{dynamic treatment regimes} refer to a setting where the same individual is observed repeatedly over time and the level as well as the type of the treatment is adjusted according to the individual's needs. References to this setting include \cite{robins1997causal}, \cite{lavori2000flexible}, \cite{murphy2001marginal}, \cite{murphy2003optimal} and \cite{murphy2005experimental}. 

The remainder of the paper is organized as follows. Section \ref{Sec:NoCov} considers a setting where the treatment outcomes do not depend on observable individual specific characteristics. Next, Section \ref{Sec:Cov} introduces covariates and establishes regret bounds for the sequential treatment policy. When grouping individuals in a specific way, these bounds are minimax optimal. It is also shown that the expected number of sub-optimal assignments increases slowly and we investigate how to handle discrete covariates. Section \ref{Sec:Delay} investigates the effect of outcomes being observed with delay. Finally, Section \ref{sec:conc} concludes while \ref{Sec:Appendix} contains all proofs. 


\section{The treatment problem without covariates}\label{Sec:NoCov}
We begin by considering the sequential treatment problem where the distributions of treatment outcomes do not depend on observable individual specific characteristics. 
While this setting may often be too restrictive, the regret bounds established in this section will be used as ingredients in establishing the properties of our treatment rules in the setting where covariates are observed on each individual prior to the treatment assignment.

Consider a setting with $K+1$ different treatments and $N$ assignments\footnote{We consider a setting with $K+1$ treatments for purely notational reasons since it is the number of suboptimal treatments, $K$, which will enter our regret bounds as well as many of the arguments in the appendix.}. $N$ is a random variable whose value need not be known to the policy maker at the beginning of the treatment assignment problem. For example, at the beginning of the year, he does not know how many will become unemployed during the year. Let $Y_t^{(i)}\in[0,1]$ denote the outcome from assigning treatment $i,\ i=1,...,K+1$ to individual $t,\ t=1,...,N$ where the subscript $t$ indicates the order in which individuals are treated. It is merely for technical reasons that we assume the treatment outcomes to take values in $[0,1]$ and this interval can be generalized to any interval $[I_1,I_2]$ for some $I_1,I_2\in\mathbb{R},\ I_1\leq I_2$ or $Y_t^{(i)}$ being sub-gaussian without qualitatively changing our results. The framework accommodates treatments with different costs since, whenever it makes sense, $Y_t^{(i)}$ can be defined net of costs. 

We allow for the data to arrive in $M$ batches of sizes $m_j,\ j=1,...,M$, such that the total number of assignments $N=\sum_{j=1}^Mm_j$. If an unemployment program is run for twelve months and new programs start every month then $M=12$ and the $m_j$ indicate how many individuals become unemployed in the $j$th month. The $m_j$ are allowed to be random variables as the policy maker does not a priori know how many will become unemployed each month. This is in contrast to typical treatment allocation problems where the size as well as the composition of the data set are taken as given. Every individual $t$ belongs to exactly one of the batches. For each batch the outcomes of the assignments  are only observed at the end of the batch. Thus, the treatment assignments for individuals belonging to batch $\tilde{j}$ can only depend on the outcomes observed from previous batches $j=1,...,\tilde{j}-1$. This is reasonable as information gained from treating persons who have become unemployed prior to person $t$, yet in the same month/batch, cannot be used to inform the treatment allocation of person $t$ as all persons from the same batch start their programs at the same time. 

For each $t=1,...,N$ the treatment outcomes can be arbitrarily correlated in the sense that we put no restrictions on the dependence structure of the entries of the vector $Y_t=(Y_t^{(1)},...,Y_t^{(K+1)})$, i.e.\ the joint distribution of the entries of $Y_t$ is left unspecified. This in accordance with real applications where an unemployed individual's response to two types of job training programs may be highly correlated. As individuals arrive independently, we assume the $Y_t$ are i.i.d.

Compared to the existing literature a distinguishing feature of our work is that we consider general \textit{welfare functions} $f:\mathbb{R}^2\to\mathbb{R}$ of the mean $\mu^{(i)}=\mathbb{E}Y_t^{(i)}$ and variance $(\sigma^2)^{(i)}=\mathbb{E}(Y_t^{(i)}-\mu^{(i)})^2$ of the treatment outcome $Y_{t}^{(i)}$. This is in contrast to most other work which only considers the expected \textit{effect} of a treatment which amounts to only considering welfare functions depending on the mean. However, it is often very important to also take into account the risk of a treatment. 

Defining $f^{(i)}=f(\mu^{(i)},(\sigma^2)^{(i)})$ the \textit{welfare maximizing} (best) treatment is denoted by $*$ and satisfies $f^{(*)}=\arg\max_{1\leq i\leq K+1}f^{(i)}$ \footnote{We assume without loss of generality that the best treatment is unique.}. The welfare maximizing treatment strikes the optimal balance between expected treatment outcome and the riskiness of the treatment. Let $\Delta_i=f^{(*)}-f^{(i)}\geq 0$ be the difference between the best and the $i$th treatment and assume that $\Delta_1\geq...\geq \Delta_K> \Delta_*=0$. The ranking of the $\Delta_i$ is without loss of generality and does not necessarily imply a ranking of neither the $\mu^{(i)}$ nor the $(\sigma^2)^{(i)}$.

A treatment allocation rule is a sequence of (random) functions $\pi=\cbr[0]{\pi_t}$ assigning a treatment from the set $\cbr[0]{1,...,K+1}$ to every individual $t=1,...,N$ . This allocation can only depend on the outcomes from previous batches. 
%

Our goal is to provide a rule $\pi$ that maximizes expected cumulated welfare over the $N$ treatments. This is equivalent to minimizing the expected difference to the infeasible welfare that would have been obtained from always assigning the best treatment $*$, i.e. minimizing the expected value of the \textit{regret}
\begin{align}
R_{N}(\pi)
=
\sum_{t=1}^{N}\left(f^{(*)}-f^{(\pi_t)}\right)
=
\sum_{j=1}^{M}\sum_{i=1}^{m_j}\left(f^{(*)}-f^{(\pi_{j,i})}\right)
\label{nocovregret}.
\end{align} 
where the second equality is due to the fact that each individual $t$ can be uniquely identified with an assignment $i$ made in a batch $j$; the assignment rule can also be written as $\pi_{j,i}$ for $j\in\cbr{1,...,M}$ and $i\in \cbr{1,...,m_j}$. 

\subsection{Examples}
Throughout this paper we assume that $f$ is Lipschitz continuous from $[0,1]^2$ equipped with the $\ell_1$-norm to $\mathbb{R}$ with Lipschitz constant $\mathcal{K}>0$, i.e.
\begin{align*}
|f(u_1,u_2)-f(v_1,v_2)|
\leq
\mathcal{K}\del[1]{|u_1-v_1|+|u_2-v_2|} 
\end{align*}
By making concrete choices for  $f$ our framework contains the following instances as special cases.
\begin{enumerate}
\item $f(\mu, \sigma^2)=\mu$ (implying $\mathcal{K}=1$) amounts to the classic bandit problem where one only targets the mean. However, unlike this paper, the classic setting does not consider batched data or outcomes that are observed with delay. 
\item $f(\mu, \sigma^2)=\frac{\mu}{\sigma}$ amounts to Sharpe ratios which are frequently used in financial applications to measure risk-return tradeoffs. If $(\sigma^2)^{(i)}\geq c$ for some $c>0$	 for all $i=1,...,K+1$ then one has by the mean value theorem that $\mathcal{K}=\max(\frac{1}{\sqrt{c}},\frac{1}{2c^{3/2}})$ works. Note that $f$ is nonlinear in $\sigma^2$ for all $\mu$. 
\item $f(\mu, \sigma^2)=-\sigma/\mu$ is the negative of the coefficient of variation in the literature on the measurement of inequality, see \cite{atkinson1970measurement}. The coefficient of variation is a measure of inequality minimizing this is encompassed by our framework. Here $\mathcal{K}=\max(\frac{1}{c^2},\frac{1}{2c^{3/2}})$ works if $\mu,\sigma^2\geq c$ for some $c>0$.
\item $f(\mu,\sigma^2)=\mu-\frac{\alpha}{2}\sigma^2$ for a risk aversion parameter $\alpha>0$ is another typical way of measuring the tradeoff between expected outcomes and their variance. Here $\mathcal{K}=\max(1,\alpha/2)$.
\item $f(\mu,\sigma^2)=-\sigma^2$ (implying $\mathcal{K}=1$) amounts to the case where one is interested only in minimizing the variance.
\item The theory developed in this paper can be extended to the case where one is interested in maximizing cumulative welfare with a welfare functions depending on any finite number of moments, i.e.  $f(\mu_1^{(i)},...,\mu_d^{(i)})$ for some $d\geq 1$, where $\mu_k^{(i)}=\mathbb{E}\sbr[1]{(Y_t^{(i)})^k}$ is the $k'th$ moment of $Y_t^{(i)}$. Higher moments than the second one may be relevant if the policy maker has, say, skewness aversion. This is relevant in dynamic portfolio allocation problems and finance as in \cite{harvey2000conditional}. To keep the exposition simple, we have chosen to focus on the case where $f$ depends on the first two moments only as the extension to welfare functions of strictly more than two moments is mainly technical.
\end{enumerate} 

\subsection{The sequential treatment policy}
Heuristically, the sequential treatment policy works by eliminating treatments that are deemed to be inferior based on the outcomes observed so far. We then take turns assigning each of the remaining treatments in the next batch. This is the exploration step. After this step, elimination takes place again.

To describe the policy more formally, let $m_{i,j}$ be the number of times treatment $i$ is assigned in batch $j$. Thus, $m_j=\sum_{i=1}^{K+1}m_{i,j}$ and we define $B_i(b)=\sum_{j=1}^bm_{i,j}$ as the number of times treatment $i$ has been assigned up to and including $b$ batches, $b=1,...,M$. Next, for a policy $\pi$ let $\hat{\mu}^{(i)}_{N_{s,i}}=\frac{1}{N_{s,i}}\sum_{t=1}^sY_t^{(i)}1_{\cbr[0]{\pi_t=i}}$ and $(\hat{\sigma}^2_{N_{s,i}})^{(i)}=\frac{1}{N_{s,i}}\sum_{t=1}^s(Y_t^{(i)}-\hat{\mu}^{(i)}_{N_{s,i}})^21_{\cbr[0]{\pi_t=i}}$ with $N_{s,i}=\sum_{t=1}^s1_{\cbr[0]{\pi_t=i}}$ be estimators of $\mu^{(i)}$ and $(\sigma^2)^{(i)}$, respectively based on observing outcomes on $s\in\cbr{1,...,N}$ individuals.  

\vspace{0.5cm}

\textbf{Sequential treatment policy:} 
Denote by $\hat{\pi}$ the sequential treatment policy. Let $\mathcal{I}_b\subseteq\cbr[0]{1,...,K+1}$ be the set of remaining treatments before batch $b$ and let $\underline{B}(b)=\min_{i\in\mathcal{I}_b}B_i(b)$ be the number of times that each remaining treatment at least has been assigned up to and including batch $b$.
\begin{enumerate}
\item In each batch $b=1,...,M$ we take turns assigning each remaining treatment. We first assign any treatments that have been assigned fewer times than any of the other remaining treatment(s). 
Thus, the difference between the number of times that any pair of remaining treatments has been assigned at the end of a batch is at most one.
\item At the end of batch $b$ eliminate treatment $\tilde{i}\in \mathcal{I}_b$ if
\begin{align*}
\max_{i\in \mathcal{I}_b} f(\hat{\mu}^{(i)}_{\underline{B}(b)},(\hat{\sigma}^2_{\underline{B}(b)})^{(i)})-f(\hat{\mu}^{(\tilde{i})}_{\underline{B}(b)},(\hat{\sigma}^2_{\underline{B}(b)})^{(\tilde{i})})\geq 32\gamma\sqrt{\frac{2}{\underline{B}(b)}\overline{\log}\left(\frac{T}{\underline{B}(b)}\right)}
\end{align*}
where $\gamma>0,\ T\in\mathbb{N}$ and $\overline{\log}(x)=\log(x)\vee 1$.
\end{enumerate}

\vspace{0.5cm}

The sequential treatment policy uses the sample counterparts of $\mu^{(i)}$ and $(\sigma^2)^{(i)}$ to evaluate whether treatment $i$ is inferior to the best of the remaining treatments. Concrete choices of $\gamma$ and $T$ guaranteeing low regret are given in Theorem \ref{NoCovRegret} and we provide some initial intuition here. The parameter $\gamma$ controls how aggressively treatments are eliminated. Small values of $\gamma$ make it easier to eliminate inferior treatments but also induce a risk of potentially eliminating the best treatment. The exact form of the elimination threshold comes from the fact the sample moments concentrate at rate $1/\sqrt{\underline{B}(b)}$ around their population counterparts. The parameter $T$, which will often be set equal to the expected sample size $n=\mathbb{E}(N),$ is needed exactly to ensure that we are cautious eliminating treatments after the first couple of batches where $\hat{\mu}^{(i)}_{\underline{B}(b)}$ and $(\hat{\sigma}^2_{\underline{B}(b)})^{(i)}$ could be based on few observations and thus need not be precise estimates of $\mu^{(i)}$ and $(\sigma^2)^{(i)}$, respectively \footnote{We are slightly more cautious than $1/\sqrt{\underline{B}(b)}$. On the other hand, one does not want to be too cautious either since this results in slow elimination of suboptimal treatments.}. From a technical point of view, this ensures that we can uniformly (over treatments) control the probability of eliminating the best treatment. Note that eliminating the best treatment is very costly as regret will accumulate linearly after such a mistake\footnote{If the best treatment is eliminated then the regret from each subsequent treatment is $f^{(*)}-f^{(\hat{\pi}_t)}\geq \Delta_K>0$}. Furthermore, the sequential treatment policy need neither to know sample size $N$, nor the number of batches $M$ in order to run. It can be stopped at any point in time with regret bounds as outlined in Theorem \ref{NoCovRegret} below.

\begin{remark}
In practice one may also consider a policy which allows treatments to reenter the treatment set even after they have been eliminated. On the other hand, there is no reason for this from a theoretical point of view as the rates in Corollary \ref{Cor:SimpleBins} below are minimax optimal such that one can at most expect to improve the constant entering the upper bound on expected regret. Heuristically, the sequential treatment policy is constructed in such a way that treatments are only eliminated if we are very certain that they are suboptimal. Thus, in this sense, there is no need to reintroduce previously eliminated treatments.
\end{remark}

\subsection{Optimal treatment assignment without covariates}
Without an upper bound on the size of the batches it is clear that no non-trivial upper bound on regret can be established. For example, the data could arrive in one batch of size $N$ implying that feedback is never received prior to any assignment. Thus, we shall assume that no batch is larger than $\overline{m}$ where $\overline{m}$ is non-random, i.e. $m_j\leq \overline{m}$ for $j=1,...,M$. Our first result provides an upper bound on the regret incurred by the sequential treatment policy.
\begin{theorem}\label{NoCovRegret}
Consider a treatment problem with $(K+1)$ treatments and an unknown number of assignments $N$ with expectation $n$ that is independent of the treatment outcomes. By implementing the sequential treatment policy with parameters $\gamma=\mathcal{K}$ and $T=n$ one obtains the following bound on the expected regret
\begin{align}
\mathbb{E}\left[R_N(\hat{\pi})\right]
\leq
C \min\left(\overline{m}\mathcal{K}^2\sum_{j=1}^{K}\frac{1}{\Delta_{j}}\overline{\log}\left(\frac{n\Delta_{j}^2}{\mathcal{K}^2}\right),\sqrt{n\mathcal{K}^3\overline{m}K\log(\overline{m}K/\mathcal{K})}\right)\label{corolLip2}
\end{align}
for a positive constant $C$. 
\end{theorem}
The upper bound in Theorem \ref{NoCovRegret} consists of two parts. The first part is adapting to the unknown distributional characteristics $\Delta_j$. Note that the regret in this part only increases logarithmically in the the expected number of treatments $n$. This logarithmic rate is unimprovable in general since it is known to be optimal even in the case where one only targets the mean (which in our setting corresponds to $f(x,y)=x)$ such that $\mathcal{K}=1$) and the treated individuals arrive one-by-one ($\overline{m}=1$), see e.g.\ Theorem 2.2 in \cite{bubeck2012regret}. On the other hand, the first part of (\ref{corolLip2}) can be made arbitrarily large by letting e.g.\ $\Delta_1\rightarrow 0$. Thus, the bound is not uniform in the underlying distribution of the data. The second part of (\ref{corolLip2}) is uniform over all $(K+1)$ tuples of distributions on $[0,1]$ and in fact yields the minimax optimal rate up to a factor of $\sqrt{\log(K)}$ even in the case where only the welfare function $f(x,y)=x$ is considered and $\overline{m}=1$. It is reasonable that both parts of the upper bound in (\ref{corolLip2}) are increasing in $\overline{m}$ since as the maximum batch size increases the time between potential elimination of suboptimal treatments increases implying that these are assigned more often. Similarly, more experimentation between treatments takes place when the number of these, $K+1$, is increased which results in increased regret.

Note that the implementation of the sequential treatment algorithm requires knowledge of the expected number of individuals that are going to be treated. In  medical experiments the total number of individuals participating is often determined a priori making $N$ known and deterministic (and equal to $n$). On the other hand, when allocating unemployed to treatments the total number of individuals becoming unemployed in the course of the year is unknown. However, one often has a good estimate of the expected value $n$ which is what matters for the treatment policy. For example, one may use averages of the number of individuals who have become unemployed in previous years to estimate $n$. Alternatively, one can use the doubling trick which resets the treatment policy at prespecified times in order to avoid any assumptions on the size of $N$ or $n$. Usage of the doubling trick would imply that eliminated treatments reappear and get another chance every time the policy is reset thus allowing for the efficiency of treatments to vary over time. For further details on the doubling trick and its implementation we refer to \cite{shalev2012online}.   

\subsection{Suboptimal treatments}
Theorem \ref{NoCovRegret} showed that the expected cumulated welfare of the sequential treatment policy will not be much smaller than the one from the infeasible policy that always assigns the best treatment. However, for an assignment rule to be ethically and politically viable it is important that it does not yield high welfare at the cost of maltreating certain individuals by wild experimentation. For example, it may not be ethically defendable for a doctor to assign a suboptimal treatment to a patient in order to gain more certainty for future treatments. The following theorem shows that the sequential treatment policy does not suffer from such a problem in the sense that the expected number of times \textit{any} suboptimal treatment is assigned only increases logarithmically in the sample size.
\begin{theorem}\label{Thm3.2}
Suppose the sequential treatment policy is implemented with parameters $T=n$ and $\gamma=\mathcal{K}$. Let $T_i(t)$ denote the number of times treatment $i$ is assigned by the sequential treatment policy up to and including observation $t$. Then 
\begin{align*}
\mathbb{E}\left[T_i(N)\right]
\leq
C\del[3]{\mathcal{K}^2K\frac{\log\left(\frac{n}{\mathcal{K}^2}\right)}{\Delta_i^2}+K\overline{m}+\mathcal{K}^2},
\end{align*}
for any suboptimal treatment $i\in\{1,...,K\}$ and a positive constant $C$. 
\end{theorem}
The important ethical guarantee on the treatment rule is that it only assigns very few persons to a suboptimal treatment (logarithmic growth rate in the sample size). It is in line with intuition that the closer any suboptimal treatment is to being optimal ($\Delta_i$ closer to zero) the more difficult it is to guarantee that this treatment is rarely assigned. The reason is that this treatment must be assigned more often before it confidently can be concluded  that it is suboptimal and thus eliminated. On the other hand, the regret incurred by assigning such a treatment is low exactly because $\Delta_i$ is small such that the increased amount of experimentation does not necessarily lead to high regret. 

\subsection{How much has been learned: out of sample performance}
So far we have considered the performance of our sequential treatment policy on the $N$ individuals being treated. However, one may also ask how much has been learned in the course of the $N$ assignments. Or, put differently, how well can we expect to treat ``out of sample''-individuals. Thus, imagine an $(N+1)$st individual to whom a treatment $I_{N}\in\cbr[0]{1,...,K+1}$ must be assigned in order to maximize the expected welfare of said individual. To be precise, the goal is to choose $I_n$ to minimize the expected value of
\begin{align*}
r_N:=\Delta_{I_{N}}=\sum_{i=1}^K\Delta_i1_{\cbr[0]{I_{N}=i}}
\end{align*}
Compared to the problem we have studied so far this is a \textit{pure exploitation} problem --- there are no gains from gathering more knowledge about the treatments as all that enters the objective function is the welfare of the $(N+1)st$ individual. To do so, one must first choose how to assign the treatment to the out of sample individual. Here we shall show that assigning the treatment that has been assigned most often in the exploration-exploitation phase by our sequential treatment policy $\hat{\pi}$ works well, i.e.
\begin{align}
I_{N}=\bar{i}\in\argmax_{i\in\cbr[0]{1,...,K+1}} T_i(N)\label{eq:oop}
\end{align}
with an arbitrary tie-breaker. We call this policy the \textit{out-of-sample} policy. Heuristically, the reason that this policy works well is that in the initial treatment period the sequential treatment policy operates by gradually eliminating sub-optimal treatments while at the same time ensuring that the best treatment is not eliminated. Thus, it is likely that the best treatment has been assigned most often. Note that the following result relies on the sample size $N$ being non-random and therefore equal to its expectation $n$. An analogous result can be proven for the case of random $N$ if one is willing to restrict the support of $N$.
\begin{theorem}\label{thm:oos}
Assume that $n\geq (K+1)\sbr[3]{c\frac{\mathcal{K}^2}{\Delta_i^2}\overline{\log}\del[2]{\frac{n\Delta_i^2}{4608\mathcal{K}^2}}+\overline{m}}$ for all $i\in\cbr[0]{1,...,K}$ for a constant $c>0$. Whenever $f(\mu, \sigma^2)$ depends on $\mu$ only one has $c=18$ while $c=4608$ whenever $f(\mu,\sigma^2)$ depends on $\mu$ and $\sigma^2$. Implementing the sequential treatment policy with parameters $T=n$ and $\gamma=\mathcal{K}$ on the first $n$ individuals and then using the out-of-sample policy in (\ref{eq:oop}) to treat individual $n+1$ yields
\begin{align*}
\E(r_n)
\leq
CK\min\del[3]{\sum_{i=1}^K\frac{\frac{\mathcal{K}^2}{\Delta_i}\overline{\log}(\frac{n\Delta_i^2}{\mathcal{K}^2})+\Delta_i\bar{m}}{n}, \sqrt{\frac{\mathcal{K}^2\overline{\log}(\frac{n}{\mathcal{K}^2})}{n}}+\frac{K\bar{m}}{n}}
\end{align*}
for a universal constant $C>0$.
\end{theorem}
The requirements on the constant $c$ can likely be refined but we focus on the rate on the upper bounds on $\E(r_N)$ here. As in Theorem \ref{NoCovRegret} the bound in Theorem \ref{thm:oos} consists of a distribution dependent and a uniform part (where the uniformity is over all $(K+1)$ tuples of distributions on $[0,1]$). Both parts witness that even though the sequential treatment policy assigns treatments in order to maximize the welfare of the $n$ treated individuals as opposed to maximizing the information by the end of $n$ treatments enough is learned to construct a policy that guarantees high welfare out of sample. This possibility is due to the fact that learning about the available treatments is inherent to maximizing the welfare over the treated individuals. In particular, both parts of the upper bound in Theorem \ref{thm:oos} tend to zero as $n\to\infty$.

\section{Treatment outcomes depending on covariates}\label{Sec:Cov}
So far we have considered the case where the outcome of a treatment does not depend on the characteristics of the individual it is assigned to. In reality, however, different persons react differently to the same type of treatment: while a certain medicine may work well for one person it may be outright dangerous to assign it to another person if this person is allergic to some of its substances. Similarly, the effect of further education on the probability of an unemployed individual finding a job may also depend on, e.g., the age of the individual: individuals close to the retirement age may benefit more from short courses updating their skill set while young individuals may benefit more from going back to school for an extended period of time.

Prior to assigning individual $t$ to a treatment we observe a vector $X_t\in[0,1]^d$ of covariates with distribution $\mathbb{P}_X$. In the case of assigning unemployed persons to various unemployment programs $X_t$ could include age, length of education, and years of experience. It is merely for technical convenience that we assume the variables to take values in $[0,1]$ and the assumption of bounded support can be replaced by tail conditions on the distribution of $X_t$. $\mathbb{P}_X$ is assumed to be absolutely continuous with respect to the Lebesgue measure with density bounded from above by $\bar{c}>0$. This rules out discrete covariates which may be very relevant in practice. In Section \ref{Sec:disc} we shall show how policies with low regret in the presence of discrete covariates can be constructed. As we now observe covariates on each individual prior to the treatment assignment we condition on these as, for example, the risk of a treatment may be individual specific and depend on, e.g., whether the person has an allergy or not. Thus, in close analogy to the setting without covariates, we now define the conditional means and variances $\mu^{(i)}(X_t)=\mathbb{E}(Y^{(i)}|X_t)$ and $(\sigma^2)^{(i)}(X_t)=\mathbb{E}\sbr[1]{(Y^{(i)}-\mu^{(i)}(X_t))^2|X_t}$ as well as $f^{(i)}(X_t)=f(\mu^{(i)}(X_t), (\sigma^2)^{(i)}(X_t))$. As $\mu^{(i)}(X_t)$ and $(\sigma^2)^{(i)}(X_t)$ are unknown to the policy maker they must be gradually learned by experimentation. In the presence of covariates a policy $\pi=\cbr[0]{\pi_t}$ is a sequence of random functions $\pi_t: [0,1]^d\to \cbr[0]{1,...,K+1}$ where $\pi_t$ can only depend on treatment outcomes from previous batches. For any $X_t$, a social planner (oracle) who knows the conditional mean and variance functions and wishes to maximize welfare assigns the treatment\footnote{If there are several treatments achieveing the maximal welfare the oracle assigns any of these.}
\begin{align*}
\pi^\star(X_t)\in\argmax_{i=1,...,K+1} f^{(i)}(X_t)
\end{align*}
and receives $f^{\left(\pi^\star(X_t)\right)}(X_t)=\max_{i=1,...,K+1}f^{(i)}(X_t)=:f^{(\star)}(X_t)$. Thus, $f^{(\star)}(x)$ is the pointwise maximum of the $f^{(i)}(x)$, $i=1,...,K+1$. The goal of a treatment policy is to get as close to the oracle solution as possible in terms of welfare. The welfare loss (regret) of a policy $\pi$ compared to the oracle is 
\begin{align}
R_N(\pi)
=
\sum_{t=1}^N\del[1]{f^{\left(\pi^\star(X_t)\right)}(X_t)-f^{\left(\pi_t(X_t)\right)}(X_t)}
=
\sum_{t=1}^N\del[1]{f^{(\star)}(X_t)-f^{\left(\pi_t(X_t)\right)}(X_t)}\label{eq:regret}
\end{align}
It is important to note the difference between equation (\ref{nocovregret}) and (\ref{eq:regret}). While (\ref{nocovregret}) considers the difference between unconditional moments (\ref{eq:regret}) considers the difference between conditional moments. The latter is more ambitious as we consider each individual separately through $X_t$ and seek to minimize the distance to the treatment that would have been optimal for this specific person (with covariates $X_t$). On the other hand, in the setting without covariates, we only seek to get as close to the outcome of the treatment that is best on average.

In order to prove upper bounds on the regret we restrict the $\mu^{(i)}(X_t)$ and $(\sigma^2)^{(i)}(X_t)$ to be reasonably smooth. This is a sensible property to impose since individuals with similar characteristics can be expected to react similarly to the same treatment. In particular, we assume that $\mu^{(i)}(X_t)$ and $\sigma^{(i)}(X_t)$ are $(\beta,L)-$H{\"o}lder continuous. To be precise, letting $\enVert[0]{\cdot}$ denote the Euclidean norm on $[0,1]^d$, we assume that $\mu^{(i)},\ (\sigma^2)^{(i)}\in \mathcal{H}(\beta, L)$ for all $i=1,...,K+1$, where $\mathcal{H}(\beta, L)$ is characterised by being those $g:[0,1]^d\to[0,1]$ such that there exist $\beta\in(0,1]$ and $L>0$ such that
\begin{align*}
\envert[1]{g(x)-g(y)}\leq L\enVert[0]{x-y}^\beta \qquad \text{for all }x,y\in [0,1]^d.
\end{align*}

\subsection{Grouping individuals}\label{Subsec:Group}
In the presence of covariates the idea of the sequential treatment policy is to group individuals into groups according to the values of the covariates. Thus, we define a partition of $[0,1]^d$ which consists of Borel measurable sets $B_1,...,B_F$, called groups/bins, such that $\mathbb{P}_X(B_j)>0$, $\cup_{j=1}^FB_j=[0,1]^d$, and $B_j\cap B_k=\emptyset$ for $j\neq k$. The policy maker groups individuals according to the value of their covariates and seeks to treat each group in a welfare maximizing way. However, the policy maker may be constrained by political or ethical considerations in his choice of grouping individuals. For example, a realistic unemployment policy cannot group individuals into overly many groups and the rules determining which group an individual belongs to cannot be too complicated. Most realistic policies would choose the groups in such a way that individuals with similar characteristics belong to the same group as it can be expected that the same policy is best for similar individuals. Figure \ref{FigBin} illustrates various ways of grouping individuals. 

\begin{figure}
\begin{center}
\includegraphics[height=4cm, width=4cm]{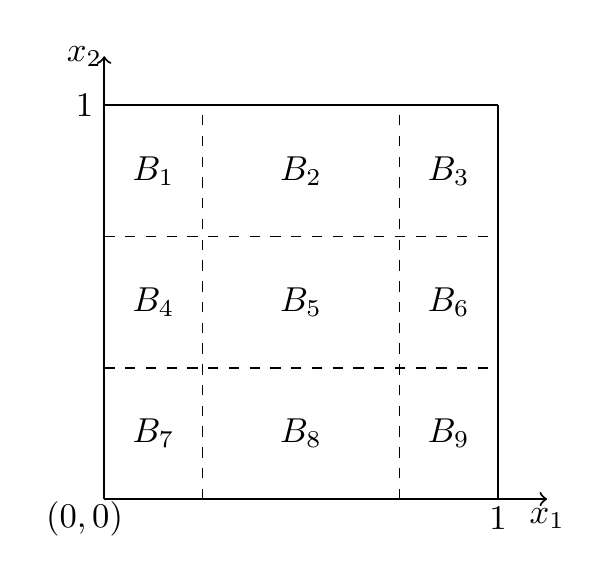}
\includegraphics[height=4cm, width=4cm]{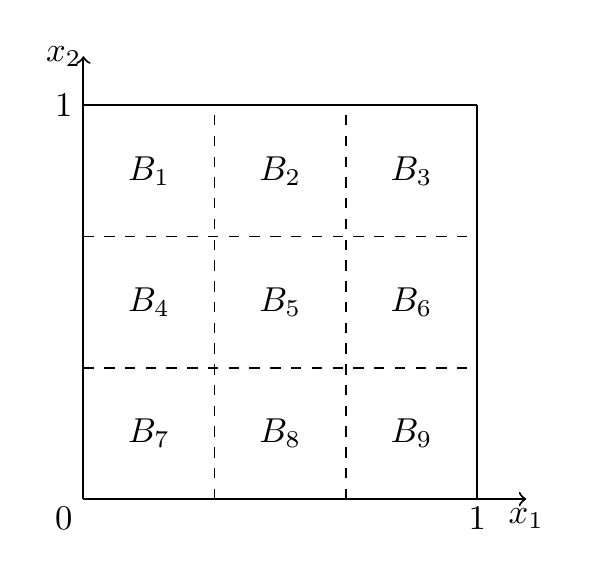}
\includegraphics[height=4cm, width=4cm]{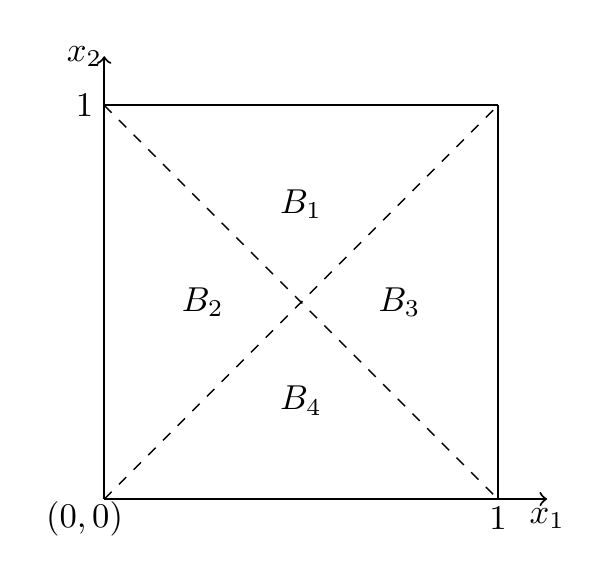}
\includegraphics[height=4cm, width=4cm]{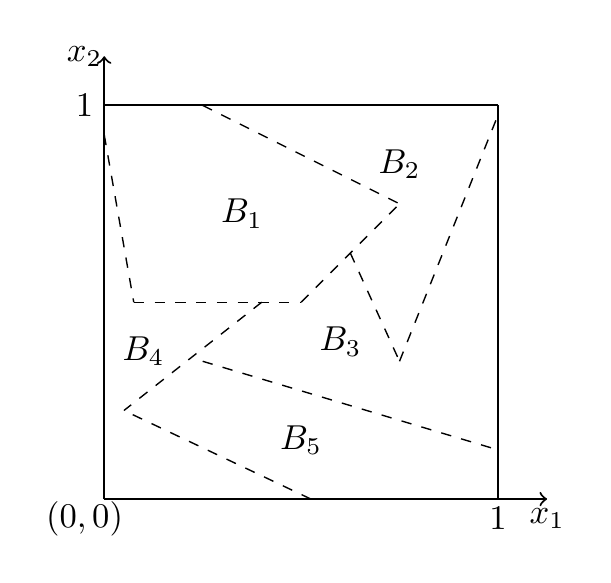}
\caption{\small Four examples of partitioning $[0,1]^d$ for $d=2$. The two leftmost ways of grouping individuals correspond to simple rules where group membership is detmermined by checking whether $x_1$ and $x_2$ are above or below certain values. The third rule corresponds to the intersection of two linear eligibility scores $a_i+b'_ix\geq c_i,\ i=1,2$. The fourth grouping, though not very practically applicable, serves to illustrate that in principle our theory allows for very general ways of grouping individuals.}
\label{FigBin}
\end{center}
\end{figure}
For any group $B_j$ define
\begin{align*}
\bar{\mu}^{(i)}_j=\mathbb{E}(Y_{t}^{(i)}|X_t\in B_j)=\frac{1}{\mathbb{P}_X(B_j)}\int_{B_j}\mu^{(i)}(x)d\mathbb{P}_X(x)
\end{align*}
and
\begin{align*}
(\bar{\sigma}^2)^{(i)}_j=Var(Y_t^{(i)}|X_t\in B_j)=\mathbb{E}({Y_t^{(i)}}^2|X_t\in B_j)-[\mathbb{E}(Y_t^{(i)}|X_t\in B_j)]^2
\end{align*}
as the mean and variance of $Y_{t}^{(i)}$ given that $X_t$ falls in $B_j$. We apply the sequential treatment policy without covariates separately to each group. To do so, define the groupwise counterpart of the welfare pertaining to treatment $i$ from the setting without covariates in Section \ref{Sec:NoCov} as $f_j^{(i)}=f(\bar{\mu}^{(i)}_j,(\bar{\sigma}^2)^{(i)}_j)$. As $\bar{\mu}^{(i)}_j$ and $(\bar{\sigma}^2)^{(i)}_j$ vary across groups one can target different optimal treatments for each group, $j=1,...,F$. We use the sequential treatment policy without covariates of Section \ref{Sec:NoCov} to target $\max_{1\leq i\leq K+1}f(\bar{\mu}^{(i)}_j,(\bar{\sigma}^2)^{(i)}_j)$ for each group. By the smoothness assumptions on $f, \mu^{(i)}(x)$ and $(\sigma^2)^{(i)}(x)$,  
$\max_{1\leq i\leq K+1}f(\bar{\mu}^{(i)}_j,(\bar{\sigma}^2)^{(i)}_j)$ will not be very far from the "fully individualized" target $f^{(\star)}(x)=\max_{1\leq i\leq K+1}f(\mu^{(i)}(x), (\sigma^2)^{(i)}(x))$ for any $x\in B_j$ as formalized in the appendix. At this stage one may ask  why one does not simply use a treatment policy which directly targets $\max_{1\leq i\leq K+1}f(\mu^{(i)}(x), (\sigma^2)^{(i)}(x))$. First, full individualization/discrimination is often not possible due to ethical or legislative constraints. Second, the regret bound obtained in Corollary \ref{Cor:SimpleBins} for the proposed policy is minimax rate optimal. Thus, even though for each group we target the policy which is best on average for that group nothing is lost (up to a multiplicative constant) even when we compare our performance to the fully individualized optimal policy. Third, a high degree of individualization is only useful for very large data sets as very small groups (in terms of the Lebesgue measure of the group) would otherwise imply very few individuals belonging to each group. This would result in only exploration being carried out for each group as no treatments can be eliminated based on very few assignments. We shall provide an example of how to optimally (in the sense of minimax regret) handle this tradeoff in Corollary \ref{Cor:SimpleBins} below. 


Let $N_{B_j}(t)=\sum_{s=1}^t1_{\cbr[0]{X_s\in B_j}}$ denote the number of individuals who have been assigned to group $B_j$ when $t$ individuals have been treated. Furthermore, $\bar{B}_j=\lambda_d(B_j)$ denotes the Lebesgue measure of group $j$. Let $\hat{\pi}_{{B_j},N_{B_j}(t)}$ be the assignment made by the sequential treatment policy without covariates applied only to individuals who belong to group $B_j$. This policy is implemented with parameters $\gamma=\mathcal{K}L$ and $T=n\bar{B}_j$. The sequential treatment treatment policy $\bar{\pi}$ with covariates is then a sequence of mappings $\bar{\pi}_t:[0,1]^d\to\cbr[0]{1,...,K+1}$ where
\begin{align*}
\bar{\pi}_t(x)=\hat{\pi}_{B_j,N_{B_j}(t)},\ \qquad x\in B_j
\end{align*}
Thus, when $X_t\in B_j$, the sequential treatment policy with covariates makes the assignment dictated by the sequential treatment policy without covariates when applied only to individuals belonging to group $B_j$.

\subsection{Upper and lower bounds on regret}
Denote by $\mathcal{S}=\mathcal{S}(\beta, L, \mathcal{K},d, \bar{c}, \overline{m})$ a treatment problem where $f$ is Lipschitz continuous with constant $\mathcal{K}$, $X_t\in[0,1]^d$ has distribution $\mathbb{P}_X$ which is absolutely continuous with respect to the Lebesgue measure with density bounded from above by $\bar{c}>0$, maximal batch size $\overline{m}$ and $\mu^{(i)},\ (\sigma^2)^{(i)}\in \mathcal{H}(\beta, L)$ for all $i=1,...,K+1$. Unless stated otherwise we will consider problems in $\mathcal{S}$ in the sequel.

The performance of our policy depends critically on the way the policy maker chooses to group individuals. To characterize this grouping, define $V_j=\sup_{x,y\in B_j}\enVert[0]{x-y}$ as the maximal possible difference in the characteristics of any two individuals assigned to group $j$. The next result provides an upper bound on the regret compared to the infeasible oracle which knows $\mu^{(i)}(x)$ and $(\sigma^2)^{(i)}(x)$ and thus whose treatment is optimal for an individual with characteristics $x\in[0,1]^d$. 

\begin{theorem}\label{RegretBoundBinsLipschitzNoMargin}
Fix $\beta\in(0,1]$, $\mathcal{K},L>0$, $d\geq 2$ and consider a treatment problem in $\mathcal{S}$. Then, for a grouping characterized by $\cbr[0]{V_1,...,V_F}$ and $\cbr[0]{\bar{B}_1,...,\bar{B}_F}$, expected regret is bounded by
\begin{align}
\mathbb{E}\left[R_N(\bar{\pi})\right]
\leq 
C\sum_{j=1}^F\left[\sqrt{\overline{m}K\log(\overline{m}K)n\bar{B}_j}+n\bar{B}_jV_j^{\beta}\right]\label{part1}
\end{align}
for a positive constant $C$. In particular, (\ref{part1}) is valid uniformly over $\mathcal{S}$.
\end{theorem}
Theorem \ref{RegretBoundBinsLipschitzNoMargin} provides an upper bound on the regret of the sequential treatment policy for \textit{any} type of grouping of individuals that the policy maker may choose. Allowing for groups with arbitrary characteristics is useful since the policy maker may be constrained in such a way that choosing the groups such that the right hand side of (\ref{part1}) is minimized over groups is not possible. The size of the regret depends on the characteristics $\bar{B}_j$ and $V_j$ of the grouping. Note that the upper bound on the regret is increasing in these two quantities. However, choosing the groups such that $\bar{B}_j$ and $V_j$ are small implies that the number of groups, $F$, must be large. In general the upper bound in (\ref{part1}) cannot be improved since by choosing the groups as in Corollary \ref{Cor:SimpleBins} below one achieves the minimax rate of regret. We elaborate further on this below.

The first part of the upper bound in (\ref{part1}) is the regret accumulated from implementing the sequential treatment policy without covariates on each group separately targeting $\max_{1\leq i\leq K+1} f(\bar{\mu}^{(i)}_j,(\bar{\sigma}^2)^{(i)}_j)$ for group $j=1,...,F$. 
 The second part of the bound in (\ref{part1}) is the approximation error resulting from targeting $\max_{1\leq i\leq K+1}f(\bar{\mu}^{(i)}_j,(\bar{\sigma}^2)^{(i)}_j)$ instead of $\max_{1\leq i\leq K+1}f(\mu^{(i)}(x), (\sigma^2)^{(i)}(x))$. Clearly, the larger the groups are chosen (as measured by $\bar{B}_j$ and $V_j$) the more dissimilar could the treatment that is best for the average individual of the group be from the treatment which is best for any given individual in the group.

A particular type of groups are the square ones which use hard thresholds for each entry of $X_t$ to create hypercubes that partition $[0,1]^d$. These are particularly relevant in practice due to their simplicity and an example of these bins is given in the second display of Figure \ref{FigBin}. More precisely, fix $P\in \mathbb{N}$ and define
\begin{align}
B_k=\cbr[2]{x\in\mathcal{X}:\frac{k_l-1}{P}\leq x_l\leq \frac{k_l}{P},\ l=1,...,d} \label{SquareBins}
\end{align}
for $k=(k_1,...,k_d)\in\{1,...,P\}^d$. Thus, $P$ is the number of splits along each dimension of $X_t$ creating a partition of $P^d$ smaller hypercubes $B_1,...,B_{P^d}$ with side lengths $1/P$. 
\begin{corollary}\label{Cor:SimpleBins}
Fix $\beta\in(0,1]$, $\mathcal{K},L>0$, $d\geq 2$ and consider a treatment problem in $\mathcal{S}$. Set $P=\lfloor\left(\frac{n}{\overline{m}K\log(\overline{m}K)}\right)^{1/(2\beta+d)}\rfloor$. Then, expected regret is bounded by
\begin{align}
\mathbb{E}\left[R_N(\bar{\pi})\right]\leq C n\left(\frac{\overline{m}K\log(\overline{m}K)}{n}\right)^\frac{\beta}{2\beta+d}\label{part2}.
\end{align}
for a positive constant $C$. In particular, (\ref{part2}) is valid uniformly over $\mathcal{S}$.
\end{corollary}
Note that the larger the number of covariates $d$, the smaller will the number of splits $P$ in each dimension be as it must be ensured that enough observations fall in each group. The larger the number of potential treatments $K+1$ is the more experimentation will take place and hence the regret compared to the infeasible oracle policy increases. 

The bound in (\ref{part2}) is, as a function of $n$, optimal in a minimax sense and cannot be improved by more than multiplicative constants. To see this consider the the case of $\overline{m}=1$ and $K=1$ (two treatments are available) such that (\ref{part2}) reduces to $E\left[R_N(\bar{\pi})\right]\leq C n^{1-\frac{\beta}{2\beta+d}}$.

\begin{theorem}\label{CovMiniMax}
Let $\overline{m}=1$ and $K=1$. For any policy $\pi$
\begin{align*}
\sup_{\mathcal{S}}\mathbb{E}\left[R_N(\pi)\right]\geq Cn^{1-\frac{\beta}{2\beta+d}}
\end{align*}
for some positive constant $C$.
\end{theorem}
Theorem \ref{CovMiniMax} shows that up to multiplicative constants no treatment policy can have a lower maximal regret over $\mathcal{S}$ than the sequential treatment policy as any policy must incur a regret at least of the same order as the sequential treatment policy.

\subsection{Ethical considerations}
We next show that even in the presence of covariates the sequential treatment policy does not make many suboptimal assignments. Our first result is a consequence of Theorem \ref{Thm3.2}. On any bin $1\leq j\leq F$  the result bounds the number of times that a treatment $1\leq i\leq K+1$ which does not maximize $f(\bar{\mu}^{(i)}_j,(\bar{\sigma}^2)^{(i)}_j)$ is assigned. Let $T_{i,j}(N)$ be the number of times treatment $i$ is assigned on bin $j$ in the course of a total of $N$ assignments. Calling treatment $i$ \textit{suboptimal on bin $B_j$} if $\Delta_i:=f_j^{(*)}-f(\bar{\mu}^{(i)}_j,(\bar{\sigma}^2)^{(i)}_j)>0$ we have the following result.

\begin{theorem}\label{Ethic1}
Fix $\beta\in(0,1]$, $\mathcal{K},L>0$, $d\geq 2$ and consider a treatment problem in $\mathcal{S}$. Then, for group $B_j$ characterized by $V_j$ and $\bar{B}_j$,
\begin{align*}
\mathbb{E}\left[T_{i,j}(N)\right]
\leq
C\del[3]{\mathcal{K}^2K\frac{\log\left(\frac{n\bar{B}_j}{\mathcal{K}^2}\right)}{\Delta_i^2}+K\overline{m}+\mathcal{K}^2},
\end{align*}
for any treatment $i$ that is suboptimal on bin $B_j$ and a positive constant $C$.
\end{theorem}
Theorem \ref{Ethic1} guarantees that any treatment whose combination of mean and variance over $B_j$ does not maximize $f$ will only rarely be assigned. In fact, the number of times a treatment that is suboptimal on bin $B_j$ is assigned only grows logarithmically in the expected number of individuals belonging to bin $B_j$. Notice the similarity to Theorem \ref{Thm3.2} where $n$ has now been replaced by $n\bar{B}_j$ which up to the constant $\bar{c}$ is an upper bound on the expected number of individuals falling in group $j$. 

A potential shortcoming of Theorem \ref{Ethic1} is that the for each group $B_j$ the maximizer of $f(\bar{\mu}^{(i)}_j,(\bar{\sigma}^2)^{(i)}_j)$ depends on the way the policy maker has chosen $B_j$. A different way of assessing the number of suboptimal treatments assigned is to consider each person individually and check whether the optimal treatment was assigned or not to this person. We say that treatment $i$ is \textit{suboptimal for individual $t$} if $f^{(\star)}(X_t)> f^{(i)}(X_t)$.  Therefore, another way of declaring the fairness of a policy $\pi$ is to provide an upper bound on the number of individuals to whom a suboptimal treatment was assigned:
\begin{align*}
S_N(\pi)
=
\sum_{t=1}^N1_{\{f^{(\star)}(X_t)\neq f^{(\pi_t)}(X_t)\}}
\end{align*}
It is sensible that a nontrivial upper bound on $\mathbb{E}(S_N(\pi))$ (a bound less than $n$) can only be established if the best treatment is sufficiently much better than the second best --- otherwise these cannot be distinguished from each other. To formalize this notion let
\begin{align*}
f^{(\sharp)}(x)=
\left\{
	\begin{array}{ll}
		\max_{i=1,...,K+1}\{f^{(i)}(x):f^{(i)}(x)<f^{(\star)}(x)\}  & \mbox{if } \min_{i=1,...,K+1}f^{(i)}(x)<f^{(\star)}(x) \\
		f^{(\star)}(x) & \text{otherwise}
	\end{array}
\right.
\end{align*}
denote the value of the second best treatment for an individual with characteristics $x\in[0,1]^d$. 

\begin{assumption}[Margin condition]
We say that the margin condition is satisfied with parameter $\alpha>0$ if there exists a constant $C>0$ and a $\delta_0\in (0,1)$ such that 
\begin{align*}
\mathbb{P}\del[1]{0<f^{(\star)}(X_t)-f^{(\sharp)}(X_t)<\delta}\leq C\delta^\alpha\qquad \forall \delta\in (0,\delta_0]
\end{align*}
\end{assumption}
The margin condition limits the probability that the best and the second best treatment are very close to each other. Larger values of $\alpha$ mean that it is easier to distinguish the best and second best treatment from each other. The margin condition has been used in the literature on statistical treatment rules by \cite{Kitagawa2015} to improve the rates of their empirical welfare maximization classifier. Before this, similar assumptions had been used in the literature on classification analysis, \cite{mammen1999smooth}, \cite{tsybakov2004optimal}. \cite{CovBandit} have used the margin condition in the context of bandits. The margin condition is satisfied if, for example, $f^{(\star)}(X_t)-f^{(\sharp)}(X_t)$ has a density with respect to the Lebesgue measure which is bounded from above by a constant $a>0$. In that case we may set $C=a$ and $\alpha=1$. We refer to \cite{Kitagawa2015} for more examples of when the margin condition is satisfied. 


\begin{theorem}\label{ISR}
Fix $\beta\in(0,1]$, $\mathcal{K},L>0$, $d\geq 2$ and consider a treatment problem in $\mathcal{S}$ which also satisfies the margin condition. Then for any policy $\pi$,
\begin{align}
\mathbb{E}(S_N(\pi))\leq C n^{\frac{1}{1+\alpha}}\mathbb{E}\left[R_N(\pi)\right]^{\frac{\alpha}{1+\alpha}}\label{ISR1} 
\end{align} 
for some positive constant $C$. Using the sequential treatment policy $\bar{\pi}$ and grouping individuals as in (\ref{SquareBins}) yields
\begin{align}
\mathbb{E}(S_N(\pi))\leq C n\left[\frac{\overline{m}K\log(\overline{m}K)}{n}\right]^{\frac{\alpha\beta}{\left(1+\alpha\right)\left(2\beta+d\right)}}.\label{ISR2}
\end{align} 
\end{theorem}
(\ref{ISR1}) provides an upper bound on the expected number of times a policy $\pi$ assigns a treatment which is suboptimal for individual $t$. This is done in terms of the regret incurred by the policy. (\ref{ISR2}) considers the case of the sequential treatment policy with a particular group structure. Note that $\mathbb{E}(S_N(\pi))$ is guaranteed to grow only sublinearily in $n$. However, as $\alpha$ approaches 0, which amounts to relaxing the margin condition and making the best and second best treatments indistinguishable, the upper bound on $\mathbb{E}(S_N(\pi))$ becomes almost linear in $n$. 

\subsection{Exogenously given groups}
Sometimes the groups $B_1,...,B_F$ are dictated exogenously upon the policy maker and thus can not be chosen to maximize welfare as in the previous section. As a result we can no longer target the welfare from the fully individualized policy, $f^{(\star)}(x)$. In our context this means that we must find \textit{one} treatment which best suits all individuals in each of the prespecified groups. For individuals in group $B_j,\ j=1,...,F$ a candidate for the omnibus best treatment is $f_j^{(*)}=\argmax_{1\leq i\leq K+1} f(\bar{\mu}^{(i)}_j,(\bar{\sigma}^2)^{(i)}_j)$, i.e. the treatment maximizing the welfare of a person with average characteristics $\bar{\mu}^{(i)}_j$ and $(\bar{\sigma}^2)^{(i)}_j$. Recalling that $f_j^{(i)}=f(\bar{\mu}^{(i)}_j,(\bar{\sigma}^2)^{(i)}_j)$ and introducing the \textit{modified regret} $\tilde{R}_j(\bar{\pi})=\sum_{t=1}^{N_{B_j}(N)}\del[2]{f_j^{(*)}-f_j^{(\hat{\pi}_{B_j,t)})}}$ of group $B_j$ of the sequential treatment policy we seek to upper bound
\begin{align}
\tilde{R}_N(\bar{\pi})=\sum_{j=1}^F\tilde{R}_j(\bar{\pi})\label{modreg}.
\end{align}    
Note that (\ref{modreg}) differs from the regret in (\ref{eq:regret}) in that we no longer target the outcome of the fully individualized treatment. 
\begin{corollary}\label{RegretExogenousGroups}
Let $d\geq 2$ and consider a treatment problem where $f$ is Lipschitz continuous with constant $\mathcal{K}$, $X_t\in[0,1]^d$ has distribution $\mathbb{P}_X$ which is absolutely continuous with respect to the Lebesgue measure with density bounded from above by $\bar{c}>0$ and maximal batch size being $\overline{m}$. Then, for a grouping characterized by $\cbr[0]{V_1,...,V_F}$ and $\cbr[0]{\bar{B}_1,...,\bar{B}_F}$, one has
\begin{align}
\mathbb{E}\left[\tilde{R}_N(\bar{\pi})\right]\leq C\sum_{j=1}^F\left[\sqrt{\overline{m}K\log(\overline{m}K)n\bar{B}_j}\right]\label{part11}
\end{align}
for a positive constant $C$. In particular, (\ref{part11}) is valid uniformly over $\mathcal{S}$.
\end{corollary}
The upper bound on modified regret is identical to the one in Theorem \ref{RegretBoundBinsLipschitzNoMargin} except for the absence of the term $C\sum_{j=1}^F\bar{B}_jV_j^{\beta}$ which previously served as an upper bound on the approximation error $f^{(\star)}(x)-f_j^{(*)}$ for all $x\in B_j$. However, as the groups are now exogenously given, this approximation error is unavoidable and it no longer makes sense to target $f^{(\star)}(x)$ as we can no longer choose the characteristics of the groups $\bar{B}_j$ and $V_j^{\beta}$ such that $n\bar{B}_jV_j^{\beta}$ is small. Note also how we no longer need the Hölder continuity of  $\mu^{(i)}$ and $\sigma^{(i)}$ since there is no approximation error to control.

\subsection{Discrete covariates}\label{Sec:disc}
Until now we have assumed $\mathbb{P}_X$ to be absolutely continuos with respect to the Lebesgue measure. However, many covariates that may influence the identity of the optimal treatment are discrete. For example, gender may affect the outcome of an allocation in an unemployment program. Furthermore, we may not always observe a continuous variable perfectly as data might only be informative about which of finitely many wealth groups an individual belongs to without providing the exact, continuously scaled, wealth.

In order to accommodate discrete covariates, partition $X_t=(X_{t,D}',X_{t,C}')'$ where $X_{t,D}\in A=A_1\times...\times A_{d_D}$ contains the measurements of the $d_D$ discrete covariates. Each $A_{l}\subseteq \mathbb{N},\ l=1,...,d_D$ is finite with cardinality $|A_{l}|$. For the continuous covariates we assume $X_{t,C}\in [0,1]^{d_C}$ such that $X_t$ is $(d_D+d_C)$-dimensional. As in (\ref{eq:regret}) the regret of our treatment policy is measured against the infeasible target $f^{(\star)}(X_t)=\max_{1\leq i\leq K+1}f(\mu^{(i)}(X_t), (\sigma^2)^{(i)}(X_t))$. On the other hand, it does not make sense to assume $\mu^{(i)}(x)=\mu^{(i)}(x_D,x_C)$ or $(\sigma^2)^{(i)}(x)=(\sigma^2)^{(i)}(x_D,x_C)$ to be $(\beta,L)-$H{\"o}lder continuous in $x_D$. Thus, discrete covariates must be handled differently from continuous ones. Instead we shall now assume that for each \textit{fixed} $a\in A$ one has that $\mu_{a}^{(i)}(x_C):=\mu^{(i)}(a,x_C)$ and $(\sigma^2)_{a}^{(i)}(x_C):=(\sigma^2)^{(i)}(a,x_C)$ belong to $\mathcal{H}(\beta,L)$. Since $a$ can only take $F_D=|A|=|A_1|\cdot...\cdot|A_{d_D}|$ possible values it is without loss of generality to assume $\beta$ and $L$ not to depend on $a$. 

Our treatment policy now works by fully individualizing treatments across the discrete covariates. In other words, for any of the $F_D$ possible values of the vector of discrete covariates we implement the sequential treatment policy $\bar{\pi}$ by constructing groups only based on the continuous variables just as in Section \ref{Subsec:Group}. For each value of the discrete covariate we allow for different ways of grouping based on the continuous covariates. For example, one may want to construct different wealth groups for men and women in order to obtain, e.g., groups with equally many individuals. For each $a\in A$ let $B_{a,j},\ j=1,...,F_{a}$ be the partition of $[0,1]^{d_C}$ used. 

Formally, for each $a\in A$, let $\bar{\pi}_{t,a}$ be the sequential treatment policy with continuous covariates applied to the grouping $B_{a,j},\ j=1,...,F_{a}$. Thus, the sequential treatment policy in the presence of discrete covariates, $\tilde{\pi}$, is a sequence of mappings $\tilde{\pi}_t: A_1\times...\times A_{d_D}\times [0,1]^{d_C}\to \cbr[0]{1,...,K+1}$ where
\begin{align*}
\tilde{\pi}_t(x)=\bar{\pi}_{t,a}(x_C)=\hat{\pi}_{(\cbr[0]{a}\times B_{a,j}),N_{a,j}(t)},\ \qquad x_D=a \text{ and } x_C\in B_{a,j}
\end{align*}          
with $N_{a,j}(t)=\sum_{s=1}^t1_{\left(X_{s,D}=a,X_{s,C}\in B_{a,j}\right)}$. Denote by $\tilde{\mathcal{S}}=\tilde{\mathcal{S}}(\beta, L, \mathcal{K},d_C, \bar{c},\overline{m})$ a treatment problem where $f$ is Lipschitz continuous with constant $\mathcal{K}$, $X_{t,D}\in A$ is discrete, $X_{C,t}\in[0,1]^d$ has distribution $\mathbb{P}_X$ which is absolutely continuous with respect to the Lebesgue measure with density bounded from above by $\bar{c}$, maximal batch size $\overline{m}$ and $\mu_{a}^{(i)},\ (\sigma^2)_{a}^{(i)}\in \mathcal{H}(\beta, L)$ for all $i=1,...,K+1$ and $a\in A$. Letting $V_{a,j}=\sup_{x,y\in B_{a,j}}||x-y||$ we have that $\tilde{\pi}$ enjoys the following upper bound on regret.
\begin{theorem}\label{RegretBound_DiscCov}
Fix $\beta\in(0,1]$, $\mathcal{K},L>0$, $d\geq 2$ and consider a treatment problem in $\tilde{\mathcal{S}}$. Then, if for for each $a\in A$ individuals are grouped as $\cbr[0]{B_{a,1},...,B_{a,F_a}}$, expected regret is bounded by
\begin{align}
\mathbb{E}\left[R_N(\tilde{\pi})\right]\leq
&C\sum_{a\in A}\sum_{j=1}^{F_a}\left(\sqrt{\overline{m}K\log(\overline{m}K)n\mathbb{P}(X_{t,D}=a,X_{t,C}\in B_{a,j})}\right.\nonumber\\
&+\left. \vphantom{\sqrt{\overline{m}K\log(\overline{m}K)n\mathbb{P}(X_{t,D}=a,X_{t,C}\in B_{a,j})}}n\mathbb{P}(X_{t,D}=a,X_{t,C}\in B_{a,j})V_{a,j}^{\beta}\right).\label{eq:DiscCov}
\end{align}
for a positive constant $C$. In particular, (\ref{eq:DiscCov}) is valid uniformly over $\tilde{\mathcal{S}}$.
\end{theorem}
The upper bound on regret in (\ref{eq:DiscCov}) generalizes the upper bounds in Theorems \ref{NoCovRegret} (no covariates) and \ref{RegretBoundBinsLipschitzNoMargin} (continuous covariates only). For example, the latter follows from (\ref{eq:DiscCov}) by letting $|A|=1$ and using that $X_{t,C}$ is absolutely continuous with respect to the Lebesgue measure with density bounded from above by $\bar{c}$. Also, the case of purely discrete covariates is covered as a special case of (\ref{eq:DiscCov}). In that case the approximation error vanishes as $V_{a,j}=0$

\section{Treatment outcomes observed with delay}\label{Sec:Delay}
Oftentimes the outcome of a treatment is only observed with delay. For example, a medical doctor may choose not to measure the effect of a treatment immediately after it has been assigned as it takes time for the treatment to work. However, delaying the measurement for an extended period of time also implies that many new patients will arrive before the outcome of the previous treatment is known. Thus, the type of treatment assigned to these patients must be decided based on less information. Put differently, there is a tradeoff between getting imprecise information now and obtaining precise information later. A similar tradeoff exists when assigning unemployed to job training programs as it takes time to find a job. Therefore, it may not be advisable to measure the effect of a job training program very shortly after its termination. 

In this section we formalize this intuition by proposing the following model for treatments being observed with delay. For simplicity, we focus first on the setting without covariates. We can decompose $Y_{t}^{(i)}$ as
\begin{align*}
Y_t^{(i)}=\mu^{(i)}+\eta_t^{(i)}
\end{align*} 
where $\mathbb{E}(\eta_t^{(i)})=0$. Since $Y_t^{(i)}, \mu^{(i)}\in[0,1]$ it follows that $\eta_t^{(i)}=Y_t^{(i)}-\mu^{(i)}\in[-1,1]$. Thus, without further assumptions, the deviations of $Y_t^{(i)}$ around its mean are in $[-1,1]$. We shall model the idea of measurements becoming more precise if they are delayed by restricting this interval. To be precise, we assume that 
\begin{align}
\eta_t^{(i)}=Y_t^{(i)}-\mu^{(i)}\in[-\bar{a}_l,\bar{a}_u]\label{eq:delayint}
\end{align} 
where $\bar{a}_l,\bar{a}_u\in[0,1]$. In this section we let $\bar{a}(D)=\bar{a}_u(D)+\bar{a}_l(D)$ be a function of the number of batches $D$ the measurements are delayed by. Thus, if $\bar{a}(D)$ is a decreasing function, increasing the delay results in $Y_t^{(i)}$ being a less noisy measure of $\mu^{(i)}$. Restricting the support of $\eta_t^{(i)}$ is not the only way of modelling that measurements become more precise if they are delayed. One could also let the variance of the $\eta_t^{(i)}$ be a decreasing function of $D$. In fact, any assumption which implies stronger concentration of sample averages around the population means will suffice. As the welfare function $f$ also depends on the second moment $\mu_2^{(i)}=\mathbb{E}\sbr[1]{{Y_t^{(i)}}^2}$ and since ${Y_t^{(i)}}^2, \mu_2^{(i)}\in[0,1]$  we will model increased measurement precision of second moments due to delay as\footnote{Assuming the same lower and upper bounds in (\ref{eq:delayint}) and (\ref{eq:delayint2}) is without loss of generality as one can simply take the smallest of the lower bounds and the largest of the upper bounds as the common values.
}
\begin{align}
{Y_t^{(i)}}^2-\mu_2^{(i)}\in[-\bar{a}_l,\bar{a}_u]\label{eq:delayint2}
\end{align} 

First, we establish an upper bound on regret of the sequential treatment policy when treatment outcomes are observed with delay in the absence of covariates. 

\vspace{0.5cm}

\textbf{Sequential treatment policy}
Denote by $\hat{\pi}$ the sequential treatment policy. Let $\mathcal{I}_b\subseteq\cbr[0]{1,...,K+1}$ be the set of remaining treatments before batch $b$ and let $\underline{B}(b)=\min_{i\in\mathcal{I}_b}B_i(b)$ be the number of outcomes that have been observed for each of the remaining treatments after batch $b$. 
\begin{enumerate}
\item In each batch $b=1,...,D-1$ we take turns assigning the treatments $\cbr[0]{1,...,K+1}$. No elimination takes place as no outcomes are observed.
\item In each batch $b=D,...,M$ we take turns assigning each remaining treatment (treatments in $\mathcal{I}_b$). 
\item At the end of batch $b=D,...,M$ eliminate treatment $\tilde{i}\in\mathcal{I}_b$ if
\begin{align*}
\max_{i\in \mathcal{I}_b} f(\hat{\mu}^{(i)}_{\underline{B}(b)},(\hat{\sigma}^2_{\underline{B}(b)})^{(i)})-f(\hat{\mu}^{(\tilde{i})}_{\underline{B}(b)},(\hat{\sigma}^2_{\underline{B}(b)})^{(\tilde{i})})\geq 16\gamma\sqrt{\frac{2\bar{a}^2}{\underline{B}(b)}\overline{\log}\left(\frac{T}{\underline{B}(b)}\right)}
\end{align*}
where $\gamma>0,\ T\in\mathbb{N}$ and $\overline{\log}(x)=\log(x)\vee 1$.
\end{enumerate}

\vspace{0.5cm}

Notice how the sequential treatment policy in the presence of delay differs from the one without delay. First, no elimination takes place after the first $D-1$ batches as no treatment outcomes are observed after these. Second, the elimination rule has been slightly modified as we can now eliminate more aggressively if $\bar{a}$ is small, i.e. the treatment outcomes are less noisy measurements of the population parameters. 

\begin{theorem}[No covariates]\label{LipschitzBoundDelay}
Consider a treatment problem with $(K+1)$ treatments and an unknown number of assignments $N$ with expectation $n$ that is independent of the treatment outcomes. The treatment outcomes are observed with a delay of $D$ batches as outlined above. By implementing the sequential treatment policy with parameters $\gamma=\mathcal{K}$ and $T=n$ one obtains the following bound on the expected regret
\begin{align}
&\mathbb{E}\left[R_N(\hat{\pi})\right]\leq\notag\\
 &C \min\left(\mathcal{K}^2\bar{a}^2\sum_{i=1}^K\frac{1}{\Delta_i}\overline{\log}\left(\frac{n\Delta_i^2}{\bar{a}^2}\right)+\overline{m}(K+D),\sqrt{\mathcal{K}^3\bar{a}^3\overline{m}K\overline{\log}\left(\overline{m}K/\mathcal{K}\bar{a}\right)n}+\overline{m}(K+D)\right),
\end{align}
where $C$ is a positive constant.
\end{theorem}  
Assume that $\bar{a}=\bar{a}(D)$ is a decreasing function. Then Theorem \ref{LipschitzBoundDelay} illustrates the tradeoff between getting imprecise information now and precise information later. This tradeoff is found in the adaptive part (first part) as well as the uniform part (second part) of the upper bound on regret of the sequential treatment policy. Increasing $D$ directly increases the upper bound on regret since information is obtained later but indirectly decreases the regret via a reduced $\bar{a}$. By making a concrete choice for $\bar{a}(D)$ one can determine the optimal delay by minimizing the upper bound on regret. 
It can also be shown that the bound in Theorem \ref{LipschitzBoundDelay} reduces to the one in Theorem  \ref{NoCovRegret} when $D=0$ and $\bar{a}=1$.

We turn next to the setting with continuous covariates and treatment outcomes being observed with delay. The introduction of covariates leads to a variant of (\ref{eq:delayint}). To be precise, we  assume that
\begin{align*}
{Y_t^{(i)}}-\mu_1^{(i)}(X_t),{Y_t^{(i)}}^2-\mu_2^{(i)}(X_t)\in[-\bar{a}_l,\bar{a}_u],
\end{align*} 
where $\mu_1^{(i)}(X_t)=\E\sbr[1]{Y_t^{(i)}|X_t}$ and $\mu_2^{(i)}(X_t)=\E\sbr[1]{{Y_t^{(i)}}^2|X_t}$. As in the setting without delay, we implement the sequential treatment policy separately for each group $B_1,...,B_F$ with parameters $\gamma=\mathcal{K}L$ and $T=n\bar{B}_j,\ j=1,...,F$.

\begin{theorem}\label{Delay}
Fix $\beta\in(0,1]$, $\mathcal{K},L>0$, $d\geq 2$ and consider a treatment problem in $\mathcal{S}$ where the outcomes are observed with a delay of $D$ batches. Then, for a grouping characterized by $\cbr[0]{V_1,...,V_F}$ and $\cbr[0]{\bar{B}_1,...,\bar{B}_F}$, expected regret is bounded by
\begin{align}
\mathbb{E}\left[R_N(\bar{\pi})\right]\leq C\left(\sum_{j=1}^F\left[\sqrt{\overline{m}K\bar{a}^3\log(\overline{m}K/\bar{a})n\bar{B}_j}+n\bar{B}_jV_j^{\beta}+K\overline{m}\right]+\overline{m}D\right).
\label{partDelay}
\end{align}
for a positive constant $C$. In particular, (\ref{partDelay}) is valid uniformly over $\mathcal{S}$.
\end{theorem}

The first part of the upper bound on expected regret in (\ref{partDelay}) (the sum over the $F$ groups) is identical to the upper bound in Theorem \ref{RegretBoundBinsLipschitzNoMargin} except for the presence of $\bar{a}$. The smaller $\bar{a}$ is the smaller will this part be as the observed outcomes of the treatments will be very close to the population counterparts and the treatment that is best for each group is quickly found. As $\bar{a}$ is usually a decreasing function in $D$, the upper bound in (\ref{partDelay}) clearly illustrates the tradeoff between postponing the measurement to get precise information later and getting (imprecise) information quickly. The term under the square root holds the key to the benefit from delaying as it corresponds to the regret of a treatment problem which starts only after $D$ batches but where measurements are observed more precisely. On the other hand, the term $\overline{m}D$ is an upper bound on the regret incurred from assigned individuals blindly for $D$ batches each of which contains no more than $\overline{m}$ individuals.

\section{Conclusions}\label{sec:conc}
This paper considers a treatment allocation problem where the individuals to be treated arrive gradually and potentially in batches. The goal of the policy maker is to maximize the welfare over the $N$ treatment assignments made. As the policy maker does not know a priori about the virtues of the available treatments, he faces an exploration-exploitation tradeoff. Prior to each assignment he observes covariates on the individual to be treated thus allowing for the optimal treatment to vary across individuals. Our setup allows the welfare function not only to depend on the expected treatment outcome but also on the risk of the treatment. We show that a variant of the sequential treatment policy obtains the minimax optimal regret. This strong welfare guarantee does not come at the price of overly wild experimentation as we show that the number of suboptimal treatments only grows quite slowly in the total number of assignments made. We also establish upper bounds on the regret of the sequential treatment policy when the outcome of the treatments are observed with delay. Finally, we introduce the ``out-of-sample'' policy for treating individuals after the initial treatment period and provide upper bounds on its regret. 


\section{Appendix}\label{Sec:Appendix}
Throughout the appendix we let $C>0$ be a constant that may change from line to line. 
\subsection{Proof of Theorems \ref{NoCovRegret} and \ref{Thm3.2}}
The following lemma will lead to Theorem \ref{NoCovRegret}.
\begin{lemma}\label{LipschitzBound}
Consider a treatment problem with $(K+1)$ treatments and unknown number of assignments $N$ with expectation $n$ that is independent of the treatment outcomes. Suppose that $f$ is Lipschitz continuous with known constant $\mathcal{K}$. For any $\Delta>0$, $T>0$ and $\gamma\geq \mathcal{K}$ the expected regret from running the sequential treatment policy can then be bounded as
\begin{align}
\mathbb{E}\left[R_N(\hat{\pi})\right]\leq C\left(\frac{\gamma^2K}{\Delta}\left(1+\frac{n}{T}\right)\overline{\log}\left(\frac{T\Delta^2}{4608\gamma^2}\right)+n\Delta^{-}+\frac{n\overline{m}K}{T}\right), \label{GenThm}
\end{align}
where $\Delta^-$ is the largest $\Delta_j$ such that $\Delta_j<\Delta$ if such a $\Delta_j$ exists, and $\Delta^-=0$ otherwise.
\end{lemma}
%

\begin{proof}
Define $\epsilon_{s}=u(s,T)=32\sqrt{\frac{2}{s}\overline{\log}\left(\frac{T}{s}\right)}$ and $\hat{\Delta}_i(s)=f(\hat{\mu}_s^{(*)},(\hat{\sigma}^2_s)^{(*)})-f(\hat{\mu}^{(i)}_s,(\hat{\sigma}^2_s)^{(i)})$. Recall that if the optimal treatment as well as some treatment $i$ have not been eliminated before batch $b$ (i.e., $i, *\in \mathcal{I}_b$), then the optimal treatment will eliminate treatment $i$ if $\hat{\Delta}_i(\underline{B}(b))\geq\gamma\epsilon_{\underline{B}(b)}$, and treatment $i$ will eliminate the optimal treatment if $\hat{\Delta}_i(\underline{B}(b))\leq-\gamma\epsilon_{\underline{B}(b)}$.

To say something about when either of these two events occurs we introduce the (unknown) quantity $\tau_i^*$ which is defined through the relation
\begin{align*}
\Delta_i=48\gamma\sqrt{\frac{2}{\tau_i^*}\overline{\log}\left(\frac{T}{\tau_i^*}\right)},\qquad i=1,...,K.
\end{align*}
Since $\tau_i^*$ in general will not be an integer, we also define $\tau_i=\ceil{\tau_i^*}$. Next introduce the hypothetical batch $b_i=\min\{l:\underline{B}(l)\geq \tau_i^*\}$. It is the first batch after which we have more than $\tau_i^*$ observations on all remaining treatment. Notice that
\begin{align}
\tau_i^*&\leq \underline{B}(b_i)\leq \tau_i^*+\overline{m}\leq C\frac{\gamma^2}{\Delta_i^2}\overline{\log}\left(\frac{T\Delta_i^2}{4608\gamma^2}\right)+\overline{m},\label{11.1}\\
\tau_i&\leq \underline{B}(b_i),\label{11.2}\\
\underline{B}(b_i)&\leq \tau_i+\overline{m},\label{11.3}
\end{align}
Notice that $1\leq \tau_1\leq ...\leq \tau_K$ and $1\leq b_1 \leq ...\leq b_K$.
Define the following events:
\begin{align*}
\mathcal{A}_i&=\{\text{The optimal treatment has not been eliminated before batch }b_i \},\\
\mathcal{B}_i&=\{\text{Every treatment }j\in\{1,...,i\} \text{ has been eliminated after batch } b_j\}.
\end{align*}
Furthermore, let $\mathcal{C}_i=\mathcal{A}_i\cap \mathcal{B}_i$, and observe that $\mathcal{C}_1\supseteq...\supseteq\mathcal{C}_K$. For any $i=1,...,K$, the contribution to regret incurred after batch $b_i$ is at most $\Delta_{i+1}N$ on $\mathcal{C}_i$. In what follows we fix a treatment, $K_0$, which we will have more to say about later. Using this we get the following decomposition of expected regret:
\begin{align}
\mathbb{E}\left[R_N\left(\hat{\pi}\right)\right]&=\mathbb{E}\left[R_N\left(\hat{\pi}\right)\left(\sum_{i=1}^{K_0}1_{\mathcal{C}_{i-1}\backslash\mathcal{C}_i}+1_{\mathcal{C}_{K_0}}\right)\right]\nonumber\\
&\leq n\sum_{i=1}^{K_0}\Delta_i\mathbb{P}\left(\mathcal{C}_{i-1}\backslash\mathcal{C}_i\right)+\sum_{i=1}^{K_0}B_i(b_i)\Delta_i+n\Delta_{K_0+1}.\label{ExpRegretLip}
\end{align}
where $\mathcal{C}_0$ denotes the underlying sample space. For every $i=1,...,K$ the event $\mathcal{C}_{i-1}\backslash\mathcal{C}_i$ can be decomposed as $\mathcal{C}_{i-1}\backslash\mathcal{C}_i=\left(\mathcal{C}_{i-1}\cap \mathcal{A}_i^c\right)\cup\left(\mathcal{B}_i^c\cap\mathcal{A}_i\cap \mathcal{B}_{i-1}\right)$.
Therefore, the first term on the right-hand side of (\ref{ExpRegretLip}) can be written as
\begin{align}
n\sum_{i=1}^{K_0}\Delta_i\mathbb{P}\left(\mathcal{C}_{i-1}\backslash\mathcal{C}_i\right)=n\sum_{i=1}^{K_0}\Delta_i\mathbb{P}\left(\mathcal{C}_{i-1}\cap\mathcal{A}_i^c\right)+n\sum_{i=1}^{K_0}\Delta_i\mathbb{P}\left(\mathcal{B}_i^c\cap\mathcal{A}_i\cap \mathcal{B}_{i-1}\right).\label{FirstTermLip}
\end{align}
Notice that $\mathbb{P}\left(\mathcal{C}_{i-1}\cap\mathcal{A}_i^c\right)=0$ if $b_{i-1}=b_i$. On the event $\mathcal{B}_i^c\cap\mathcal{A}_i\cap \mathcal{B}_{i-1}$ the optimal treatment has not eliminated treatment $i$ after batch $b_i$. Therefore, for the last term on the right hand side of equation (\ref{FirstTermLip}) we find that
\begin{align*}
\mathbb{P}\left(\mathcal{B}_i^c\cap\mathcal{A}_i\cap \mathcal{B}_{i-1}\right)&\leq \mathbb{P}\left(\hat{\Delta}_i(\underline{B}(b_i))\leq \gamma\epsilon_{\underline{B}(b_i)}\right)\\
&\leq \mathbb{P}\left(\hat{\Delta}_i(\underline{B}(b_i))-\Delta_i\leq \gamma\epsilon_{\tau_i}-\Delta_i\right)\\
&\leq
 \mathbb{E}\left[\mathbb{P}\left(|\hat{\Delta}_i(\underline{B}(b_i))-\Delta_i|\geq \frac{1}{2}\gamma\epsilon_{\tau_i}|\underline{B}(b_i)\right)\right]
\end{align*}
For any $s\geq \tau_i$ we have that
\begin{align}
&\mathbb{P}\left(|\hat{\Delta}_i(s)-\Delta_i|\geq \frac{1}{2}\gamma\epsilon_{\tau_i}\right)\notag\\
&\leq \mathbb{P}\left(|f(\hat{\mu}_s^{(*)},(\hat{\sigma}^2_s)^{(*)})-f(\hat{\mu}^{(i)}_s,(\hat{\sigma}^2_s)^{(i)})+f(\mu^{(i)},(\sigma^2)^{(i)})-f(\mu^{(*)},(\sigma^2)^{(*)})|\geq \frac{1}{2}\gamma\epsilon_{\tau_i}\right)\nonumber\\
&\leq \mathbb{P}\left(|f(\hat{\mu}_s^{(*)},(\hat{\sigma}^2_s)^{(*)})-f(\mu^{(*)},(\sigma^2)^{(*)})|\geq \frac{1}{4}\gamma\epsilon_{\tau_i}\right)
+\mathbb{P}\left(|f(\hat{\mu}_s^{(i)},(\hat{\sigma}^2_s)^{(i)})-f(\mu^{(i)},(\sigma^2)^{(i)})|\geq \frac{1}{4}\gamma\epsilon_{\tau_i}\right).\label{11.4}
\end{align}
Furthermore, for any $j\in\{i,*\}$, we have
\begin{align}
&\mathbb{P}\left(|f(\hat{\mu}_s^{(j)},(\hat{\sigma}^2_s)^{(j)})-f(\mu^{(j)},(\sigma^2)^{(j)})|\geq \frac{1}{4}\gamma\epsilon_{\tau_i}\right)\notag\\
&\leq 
\mathbb{P}\left(|\hat{\mu}_s^{(j)}-\mu^{(j)}|+|(\hat{\sigma}^2_s)^{(j)}-(\sigma^2)^{(j)}|\geq \frac{1}{4\mathcal{K}}\gamma\epsilon_{\tau_i}\right)\nonumber\\
&\leq \mathbb{P}\left(|\hat{\mu}_s^{(j)}-\mu^{(j)}|\geq \frac{1}{8\mathcal{K}}\gamma\epsilon_{\tau_i}\right)+\mathbb{P}\left(|(\hat{\sigma}^2_s)^{(j)}-(\sigma^2)^{(j)}|\geq \frac{1}{8\mathcal{K}}\gamma\epsilon_{\tau_i}\right).\label{11.5}
\end{align}
By the mean value theorem we have that
\begin{align}
\mathbb{P}\left(|(\hat{\sigma}^2_s)^{(j)}-(\sigma^2)^{(j)}|\geq \frac{1}{8\mathcal{K}}\gamma\epsilon_{\tau_i}\right)\leq \mathbb{P}\left(|\hat{\mu}_s^{(j)}-\mu^{(j)}|\geq \frac{1}{32\mathcal{K}}\gamma\epsilon_{\tau_i}\right)+\mathbb{P}\left(|(\hat{\mu}_{2,s})^{(j)}-\mu_2^{(j)}|\geq \frac{1}{16\mathcal{K}}\gamma\epsilon_{\tau_i}\right),\label{11.6}
\end{align}
where $\mu_2=\mathbb{E}\left[Y_1^2\right]$ and $\hat{\mu}_{2,s}=\frac{1}{s}\sum_{i=1}^sY_i^2$. By combining equations (\ref{11.4}),(\ref{11.5}), (\ref{11.6}, and applying Hoeffding's inequality as well as the fact that $\gamma\geq\mathcal{K}$, we arrive at the following bound,
\begin{align*}
\mathbb{P}\left(|\hat{\Delta}_i(s)-\Delta_i|\geq \frac{1}{2}\gamma\epsilon_{\tau_i}\right)&\leq C\exp\left(-\frac{1}{1024}\epsilon_{\tau_i}^2 s\right)\\
&\leq C\exp\left(-\frac{1}{1024}\epsilon_{\tau_i}^2 \tau_i\right)\\
&=C\exp\left(-\overline{\log}\left(\frac{T}{\tau_i}\right)\right)\\
&\leq C\frac{\tau_i}{T}.
\end{align*}
Thus,
\begin{align}
\mathbb{P}\left(\mathcal{B}_i^c\cap\mathcal{A}_i\cap \mathcal{B}_{i-1}\right)&\leq C\frac{\tau_i}{T}\label{eq:B_icAB_i-1}
\end{align} 
On the event $\mathcal{C}_{i-1}\cap\mathcal{A}_i^c$ the optimal treatment is eliminated between batch $b_{i-1}+1$ and $b_i$. Furthermore, every suboptimal treatment $j\leq i-1$ has also been eliminated. Therefore the probability of this event can be bounded as follows:
\begin{align*}
\mathbb{P}\left(\mathcal{C}_{i-1}\cap\mathcal{A}_i^c\right)&\leq\mathbb{P}\left(\exists(j,s),i\leq j\leq K,b_{i-1}+1\leq s\leq b_i;\hat{\Delta}_j(\underline{B}(s))\leq -\gamma \epsilon_{\underline{B}(s)}\right)\\
&\leq \sum_{j=i}^K\mathbb{P}\left(\exists s,b_{i-1}+1\leq s\leq b_i;\hat{\Delta}_j(\underline{B}(s))\leq -\gamma \epsilon_{\underline{B}(s)}\right)\\
&= \sum_{j=i}^K\left[\Phi_j(b_i)-\Phi_j(b_{i-1})\right],
\end{align*}
where $\Phi_j(b)=\mathbb{P}\left(\exists s\leq b; \hat{\Delta}_j(\underline{B}(s))\leq -\gamma \epsilon_{\underline{B}(s)}\right)$. We now proceed to bound terms of the form $\Phi_j(b_i)$ for $j\geq i$. 
\begin{align*}
\mathbb{P}\left(\exists s\leq b_i; \hat{\Delta}_j(\underline{B}(s))\leq -\gamma \epsilon_{\underline{B}(s)}\right)&\leq \mathbb{P}\left(\exists s\leq b_i; \hat{\Delta}_j(\underline{B}(s))-\Delta_j\leq -\gamma \epsilon_{\underline{B}(s)}\right)\\
&\leq\mathbb{P}\left(\exists s\leq \underline{B}(b_i);\hat{\Delta}_j(s)-\Delta_j\leq -\gamma \epsilon_{s}\right)\\
&\leq\mathbb{P}\left(\exists s\leq \tau_i+\overline{m};\hat{\Delta}_j(s)-\Delta_j\leq -\gamma \epsilon_{s}\right)\\
&\leq \mathbb{P}\left(\exists s\leq \tau_i+\overline{m};|f(\hat{\mu}_s^{(j)},(\hat{\sigma}_s^2)^{(j)})-f(\mu^{(j)},(\sigma^2)^{(j)})|\geq \frac{1}{2}\gamma \epsilon_{s}\right)\\
&+ \mathbb{P}\left(\exists s\leq \tau_i+\overline{m};|f(\hat{\mu}_s^{(*)},(\hat{\sigma}_s^2)^{(*)})-f(\mu^{(*)},(\sigma^2)^{(*)})|\geq \frac{1}{2}\gamma \epsilon_{s}\right).
\end{align*}
For any $j\in\{i,...,K,*\}$ we find that
\begin{align*}
&\mathbb{P}\left(\exists s\leq \tau_i+\overline{m};|f(\hat{\mu}_s^{(j)},(\hat{\sigma}_s^2)^{(j)})-f(\mu^{(j)},(\sigma^2)^{(j)})|\geq \frac{1}{2}\gamma \epsilon_{s}\right)\\ &\leq 
\mathbb{P}\left(\exists s\leq \tau_i+\overline{m};|\hat{\mu}_s^{(j)}-\mu^{(j)}|\geq \frac{1}{4\mathcal{K}}\gamma \epsilon_{s}\right)
+\mathbb{P}\left(\exists s\leq \tau_i+\overline{m};|(\hat{\sigma}_s^2)^{(j)})-(\sigma^2)^{(j)})|\geq \frac{1}{4\mathcal{K}}\gamma \epsilon_{s}\right)\\
&\leq \mathbb{P}\left(\exists s\leq \tau_i+\overline{m};|\hat{\mu}_s^{(j)}-\mu^{(j)}|\geq \frac{1}{4\mathcal{K}}\gamma \epsilon_{s}\right)
+\mathbb{P}\left(\exists s\leq \tau_i+\overline{m};|(\hat{\mu}_{2,s})^{(j)}-\mu_2^{(j)}|\geq \frac{1}{8\mathcal{K}}\gamma \epsilon_{s}\right)\\
&+\mathbb{P}\left(\exists s\leq \tau_i+\overline{m};|\hat{\mu}_s^{(j)}-\mu^{(j)}|\geq \frac{1}{16\mathcal{K}}\gamma \epsilon_{s}\right)\\
&\leq C\frac{\tau_i+\overline{m}}{T}
\end{align*}
where we have used equation (\ref{11.3}) and Lemma A.1 in \cite{CovBandit}. It follows that
\begin{align}
\sum_{i=1}^{K_0}\Delta_i\mathbb{P}\left(\mathcal{C}_{i-1}\cap\mathcal{A}_i^c\right)&\leq \sum_{i=1}^{K_0}\Delta_i\sum_{j=i}^{K}\left[\Phi_j(b_i)-\Phi_j(b_{i-1})\right]\nonumber\\
&\leq \sum_{j=1}^{K}\sum_{i=1}^{j\wedge K_0-1}\Phi_j(b_i)\left(\Delta_i-\Delta_{i+1}\right)+\sum_{j=K_0}^K\Delta_{K_0}\Phi_j(b_{K_0})+\sum_{j=1}^{K_0-1}\Delta_j\Phi_j(b_j)\nonumber\\
&\leq \frac{48}{T}\sum_{j=1}^{K}\sum_{i=1}^{j\wedge K_0-1}\left(\tau_i+\overline{m}\right)\left(\Delta_i-\Delta_{i+1}\right)+\frac{48}{T}\sum_{j=1}^K\Delta_{j\wedge K_0}\left(\tau_{j\wedge K_0}+\overline{m}\right).\label{11.7}
\end{align} 
Using equation (\ref{11.2}) we obtain 
\begin{align*}
\sum_{j=1}^{K}\sum_{i=1}^{j\wedge K_0-1}\tau_i\left(\Delta_i-\Delta_{i+1}\right)&\leq C\gamma^2\sum_{j=1}^{K}\sum_{i=1}^{j\wedge K_0-1}\frac{\left(\Delta_i-\Delta_{i+1}\right)}{\Delta_i^2}\overline{\log}\left(\frac{T\Delta_i^2}{4608\gamma^2}\right)\\
&\leq C\gamma^2\sum_{j=1}^{K}\int_{\Delta_{j\wedge K_0}}^{\Delta_1}\frac{1}{x^2}\overline{\log}\left(\frac{Tx^2}{4608\gamma^2}\right)dx\\
&\leq C\gamma^2\sum_{j=1}^{K}\frac{1}{\Delta_{j\wedge K_0}}\overline{\log}\left(\frac{T\Delta_{j\wedge K_0}^2}{4608\gamma^2}\right).
\end{align*}
The part involving $\overline{m}$ in equation (\ref{11.7})can be bounded by
\begin{align*}
\overline{m}\sum_{j=1}^{K}\sum_{i=1}^{j\wedge K_0-1}\left(\Delta_i-\Delta_{i+1}\right)+\sum_{j=1}^K\Delta_{j\wedge K_0}\overline{m}&\leq \overline{m}K.  
\end{align*}
Bringing things together we have 
\begin{align}
\sum_{i=1}^{K_0}\Delta_i\mathbb{P}\left(\mathcal{C}_{i-1}\cap\mathcal{A}_i^c\right)&\leq C\left(\frac{\gamma^2}{T}\sum_{j=1}^K\frac{1}{\Delta_{j\wedge K_0}}\overline{\log}\left(\frac{T\Delta_{j\wedge K_0}^2}{4608\gamma^2}\right)+\frac{K\overline{m}}{T}\right)
\end{align}
Combining this with equation (\ref{FirstTermLip}) and (\ref{ExpRegretLip}) we obtain
\begin{align}
\mathbb{E}\left[R_N\left(\hat{\pi}\right)\right]&\leq C\biggl(\frac{\gamma^2n}{T}\sum_{j=1}^K\frac{1}{\Delta_{j\wedge K_0}}\overline{\log}\left(\frac{T\Delta_{j\wedge K_0}^2}{4608\gamma^2}\right)
+\frac{n\gamma^2}{T}\sum_{j=1}^{K_0}\frac{1}{\Delta_{j}}\overline{\log}\left(\frac{T\Delta_{j}^2}{4608\gamma^2}\right)\nonumber\\
&+\sum_{i=1}^{K_0}B_i(b_i)\Delta_i+n\Delta_{K_0+1}+\frac{n\overline{m}K}{T}\biggr)\nonumber\\
&\leq C\biggl(\left(1+\frac{n}{T}\right) \gamma^2\sum_{j=1}^{K_0}\frac{1}{\Delta_{j}}\overline{\log}\left(\frac{T\Delta_{j}^2}{4608\gamma^2}\right)+\frac{\gamma^2n}{T}\frac{K-K_0}{\Delta_{K_0}}\overline{\log}\left(\frac{T\Delta_{K_0}^2}{4608\gamma^2}\right)\nonumber\\
&+n\Delta_{K_0+1}+\frac{n\overline{m}K}{T}\biggr).\label{11.8}
\end{align}
Fix $\Delta>0$ and let $K_0$ be such that $\Delta_{K_0+1}=\Delta^-$. Define the function 
\begin{align*}
\phi(x)=\frac{1}{x}\overline{\log}\left(\frac{Tx^2}{4608\gamma^2}\right),
\end{align*}
and notice that $\phi(x)\leq 2e^{-1/2}\phi(x')$ for any $x\geq x'\geq 0$. Using this with $x'=\Delta$ and $x=\Delta_i$ for $i\leq K_0$ we obtain 
\begin{align}
\mathbb{E}\left[R_N(\hat{\pi})\right]\leq C\left(\frac{\gamma^2K}{\Delta}\left(1+\frac{n}{T}\right)\overline{\log}\left(\frac{T\Delta^2}{4608\gamma^2}\right)+n\Delta^{-}+\frac{n\overline{m}K}{T}\right).\label{11.9}
\end{align}
\qed
\end{proof}

\begin{proof}[Proof of Theorem \ref{NoCovRegret}]
Consider the sequential treatment policy with $\gamma=\mathcal{K}$ and $T=n$. From equation (\ref{11.8}) it follows that for any $K_0\leq K$
\begin{align}
\mathbb{E}\left[R_N(\hat{\pi})\right]
&\leq C\biggl(\mathcal{K}^2\sum_{j=1}^{K_0}\frac{1}{\Delta_{j}}\overline{\log}\left(\frac{n\Delta_{j}^2}{4608\mathcal{K}^2}\right)\notag\\
&+\mathcal{K}^2\frac{K-K_0}{\Delta_{K_0}}\overline{\log}\left(\frac{n\Delta_{K_0}^2}{4608\mathcal{K}^2}\right)
+n\Delta_{K_0+1}+\frac{n\overline{m}K}{T}\biggr).\label{corolLip1}
\end{align}
This can be used to show
\begin{align}
\mathbb{E}\left[R_N(\hat{\pi})\right]\leq C\min\left(\mathcal{K}^2\sum_{j=1}^{K}\frac{1}{\Delta_{j}}\overline{\log}\left(\frac{n\Delta_{j}^2}{4608\mathcal{K}^2}\right)+\overline{m}K,\sqrt{n\mathcal{K}^3\overline{m}K\log(\overline{m}K/\mathcal{K})}\right).\label{corolLip22}
\end{align}
where the first part of the upper bound in (\ref{corolLip22}) follows by using (\ref{corolLip1}) with $K=K_0$. The second part follows from Lemma \ref{LipschitzBound} by choosing $\Delta=\sqrt{4608(23757+\overline{m})\mathcal{K}K\log((23757+\overline{m})K/\mathcal{K})/n}$. 
\end{proof}

\begin{proof}[Proof of Theorem \ref{Thm3.2}]
The idea of the proof is similar to the one used in the proof of Theorem 1 in \cite{auer2002}. For now we will keep $N$ fixed.\footnote{In other words all calculations are done conditional on $N$.} First, note that for any positive integer $l$
\begin{align*}
T_i(N)
=
1+\sum_{t=K+1}^N\bm{1}_{\cbr[0]{\hat{\pi}_t=i}}
\leq
l+\sum_{t=K+1}^N\bm{1}_{\cbr[0]{\hat{\pi}_t=i,\ T_i(t-1)\geq l}}
\leq
l+\sum_{t=K+1}^N\bm{1}_{\cbr[0]{T_i(t-1)\geq l}}
\leq
l+N\bm{1}_{\cbr[0]{T_i(N-1)\geq l}}
\end{align*}
It remains to bound the probability of the event $\cbr[0]{T_i(N-1)\geq l}$. This is the probability that treatment $i$ has not been eliminated before having been assigned at least $l$ times. Define $\tilde{b}_i=\max\{b:\sum_{j=1}^bm_{i,j}< l\}$ and note that if treatment $i$ is assigned $l$ times then it cannot have been eliminated after $\tilde{b}_i$ batches. In particular, it cannot have been eliminated by the optimal treatment. Let 
$$\mathcal{A}_i=\cbr[0]{\text{the optimal treatment has not been eliminated after batch $\tilde{b}_i$}}$$
For any $t$ we have that $\cbr[0]{T_i(N-1)\geq l}\subseteq (\cbr[0]{T_i(N-1)\geq l}\cap \mathcal{A}_i)\cup \mathcal{A}_i^c$.  Thus, $(\cbr[0]{T_i(N-1)\geq l}\cap \mathcal{A}_i)\subseteq \cbr[1]{\hat{\Delta}_i(\underline{B}(\tilde{b}_i))-\Delta_i\leq \gamma\epsilon_{\underline{B}(\tilde{b}_i)}-\Delta_i}$ which implies
\begin{align*}
\mathbb{E}\left[T_i(N)\right]&\leq l+N\mathbb{P}\left(\cbr[0]{T_i(N-1)\geq l}\cap \mathcal{A}_i\right)+ N\mathbb{P}\left(\mathcal{A}_i^c\right)\\
&\leq l+N\mathbb{P}\left(\hat{\Delta}_i(\underline{B}(\tilde{b}_i))-\Delta_i\leq \gamma\epsilon_{\underline{B}(\tilde{b}_i)}-\Delta_i\right)+ N\mathbb{P}\left(\mathcal{A}_i^c\right)
\end{align*}
From equations (\ref{11.1}) and (\ref{11.2}) of Lemma \ref{LipschitzBound} we have that (where $\tau_i$ is defined in the said lemma)
\begin{align*}
\tau_i\leq \frac{C\mathcal{K}^2}{\Delta_i^2}\overline{\log}\left(\frac{n}{4609\mathcal{K}^2}\right)
\end{align*}
Thus, by letting $l=\bar{m}+\lceil\frac{4609\mathcal{K}^2}{\Delta_i^2}\overline{\log}\left(\frac{n}{4609\mathcal{K}^2}\right)\rceil$ it follows that $\tau_i\leq l-\bar{m}\leq \underline{B}(\tilde{b}_i)<l$. In particular, we have that $\gamma\epsilon_{\underline{B}(\tilde{b}_i)}\leq \gamma\epsilon_{\tau_i}\leq \frac{2}{3}\Delta_i$. Hence,
\begin{align*}
\mathbb{P}\left(\hat{\Delta}_i(\underline{B}(\tilde{b}_i))-\Delta_i\leq \gamma\epsilon_{\underline{B}(\tilde{b}_i)}-\Delta_i\right)&\leq \mathbb{P}\left(|\hat{\Delta}_i(\underline{B}(\tilde{b}_i))-\Delta_i|\geq \frac{1}{3}\Delta_i\right)\\
&\leq C\mathbb{E}\left[ \exp\left(-\frac{\underline{B}(\tilde{b}_i)\Delta_i^2}{1536}\right)\right]\\
&\leq C\exp\left(-\frac{(l-\bar{m})\Delta_i^2}{1536}\right)\\
&\leq C\frac{\mathcal{K}^2}{n}. 
\end{align*}
Next, we bound the term involving $\mathcal{A}_i^c$. To this end we start by noting that if the optimal treatment does not survive until batch $\tilde{b}_i$, then it must have been eliminated in one of the batches before $\tilde{b}_i$.
\begin{align}
\mathbb{P}\left(\mathcal{A}_i^c\right)&\leq \sum_{j=1}^K\mathbb{P}\left(\exists s\leq \underline{B}(\tilde{b}_i):\hat{\Delta}_j(s)\leq -\gamma\epsilon_s\right)\\
&\leq \sum_{j=1}^K\mathbb{P}\left(\exists s\leq l:\hat{\Delta}_j(s)\leq -\gamma\epsilon_s\right)\\
&\leq CK\frac{l}{n}\label{eq:Ac},
\end{align} 
where the last inequality follows from an application of Lemma A.1 in \cite{CovBandit}. Bringing things together, taken expectations with respect to $N$ and using Jensen's inequality in order to replace $N$ with its expectation yields the desired result.
\end{proof}

\begin{proof}[Proof of Theorem \ref{thm:oos}]
The notation is as in the proof of Lemma \ref{LipschitzBound}. Assume that $\bar{i}\neq *$. Clearly, $T_{\bar{i}}(n)\geq \frac{n}{K+1}\geq \sbr[1]{C\frac{\mathcal{K}^2}{\Delta_{\bar{i}}^2}\overline{\log}\del[1]{\frac{n\Delta_{\bar{i}}^2}{4608\mathcal{K}^2}}+\overline{m}}$. Thus, treatment $\bar{i}$ can not have been eliminated after batch $b_{\bar{i}}$. Therefore, for all $i=1,...,K$, 
\begin{align*}
\cbr[1]{\bar{i}=i}
\subseteq
\cbr[2]{T_{i}(n)\geq \frac{n}{K+1}}
\subseteq 
\mathcal{B}_{i}^c
\subseteq
(\mathcal{B}_{i}^c\cap\mathcal{A}_{i})\cup\mathcal{A}_{i}^c
\end{align*}
which implies that we can bound $\E(r_n)$ as
\begin{align}
\E(r_n)
\leq
\sum_{i=1}^K\Delta_i\mathbb{P}(\bar{i}=i)
\leq
\sum_{i=1}^K \Delta_i\del[1]{\mathbb{P}(\mathcal{B}_{i}^c\cap\mathcal{A}_i)+\mathbb{P}(\mathcal{A}_i^c)}\label{eq:simple_regret}
\end{align}
Let $i\in\cbr[0]{1,...,K}$ be arbitrary. Inspecting the argument leading to (\ref{eq:B_icAB_i-1}) in the proof of Lemma \ref{LipschitzBound} reveals that the upper bound there also upper bounds $\mathbb{P}(\mathcal{B}_{i}^c\cap\mathcal{A}_i)$. Thus,  $\mathbb{P}(\mathcal{B}_{i}^c\cap\mathcal{A}_i)\leq C\frac{\tau_i}{T}$. Similarly, $\mathbb{P}(\mathcal{A}_i^c)$ has been bounded by $CK\del[1]{\bar{m}+\lceil\frac{4609\mathcal{K}^2}{\Delta_i^2}\overline{\log}\del[1]{\frac{n}{4609\mathcal{K}^2}}\rceil}/n$ in the proof of Theorem \ref{Thm3.2} (see (\ref{eq:Ac})). Hence, upon using that $\gamma=\mathcal{K}$, $T=n$ as well as $\tau_i\leq C\frac{\mathcal{K}^2}{\Delta_i^2}\overline{\log}\del[2]{\frac{n\Delta_i^2}{4608\mathcal{K}^2}}+\overline{m}$ one has
\begin{align*}
\E(r_n)
\leq
CK\sum_{i=1}^K\frac{\frac{\mathcal{K}^2}{\Delta_i}\overline{\log}(\frac{n\Delta_i^2}{\mathcal{K}^2})+\Delta_i\bar{m}}{n}
\end{align*}
as claimed.

We now turn to the uniform bound. From (\ref{eq:simple_regret}) we get that for any $\Delta>0$
\begin{align*}
\E(r_n)
\leq
\sum_{i:\Delta_i\geq \Delta} \Delta_i\del[1]{\mathbb{P}(\mathcal{B}_{i}^c\cap\mathcal{A}_i)+\mathbb{P}(\mathcal{A}_i^c)}+\Delta
\end{align*}
Combining this with the upper bounds on $\mathbb{P}(\mathcal{B}_{i}^c\cap\mathcal{A}_i)$ and $\mathbb{P}(\mathcal{A}_i^c)$ used above and using $\Delta_i\leq 1$ for all $i=1,...,K$ yields that
\begin{align*}
\E(r_n)
\leq
CK\sum_{i:\Delta_i>\Delta}\frac{\frac{\mathcal{K}^2}{\Delta_i}\overline{\log}(\frac{n\Delta_i^2}{\mathcal{K}^2})+\Delta_i\bar{m}}{n}+\Delta
\leq
CK^2\frac{\frac{\mathcal{K}^2}{\Delta}\overline{\log}(\frac{n}{\mathcal{K}^2})}{n}+\Delta+ C\frac{K^2\bar{m}}{n}
\end{align*}
Minimizing the right hand side of the above display with respect to $\Delta$ results in $\Delta=CK\sqrt{\frac{\mathcal{K}^2\overline{\log}(\frac{n}{\mathcal{K}^2})}{n}}$ which upon insertion into the above display yields 
\begin{align*}
\E(r_n)
\leq
CK\sqrt{\frac{\mathcal{K}^2\overline{\log}(\frac{n}{\mathcal{K}^2})}{n}}+C\frac{K^2\bar{m}}{n}.
\end{align*}
\end{proof}

\subsection{Proof of Theorems in Section \ref{Sec:Cov}}

\begin{proof}[Proof of Theorem \ref{RegretBoundBinsLipschitzNoMargin}] 
It is convenient to define the constant $c_1=6L\mathcal{K}+1$, which will enter several of the bounds derived below. Furthermore, we let $c$ denote a positive constant which may change from line to line. By the construction of the treatment policy it follows that the regret can be written as $R_N(\bar{\pi})=\sum_{j=1}^FR_j(\bar{\pi})$, where
\begin{align*}
R_j(\bar{\pi})=\sum_{t=1}^N\left(f^{(\star)}(X_{t})-f^{(\hat{\pi}_{B_j},N_{B_j}(t))}(X_{t})\right)1_{\left(X_{t}\in B_j\right)}.
\end{align*}
We start by providing an upper bound on the welfare lost for each group $B_j$ due to the policy targeting $f_j^{(*)}=\max_{1\leq i\leq K} f(\bar{\mu}^{(i)}_j, (\bar{\sigma}^2)^{(i)}_j)$ instead of $f^{(\star)}(x)$. To this end note that
\begin{align*}
f^{(\star)}(x)&=\max_{1\leq i\leq K+1} f(\mu^{(i)}(x), (\sigma^2)^{(i)}(x))\\
&\leq
\max_{1\leq i\leq K+1} f(\bar{\mu}^{(i)}_j, (\bar{\sigma}^2)^{(i)}_j)
+
\mathcal{K}\max_{1\leq i\leq K+1}|\mu^{(i)}(x)-\bar{\mu}^{(i)}_j|+\mathcal{K}\max_{1\leq i\leq K+1}|(\sigma^2)^{(i)}(x)-(\bar{\sigma}^2)^{(i)}_j|.
\end{align*}
Fix $x\in B_j$ and $i\in\cbr[0]{1,...,K+1}$. Then, for all $y\in B_j$, one has by the ($L,\beta$)-H\"{o}lder continuity of $\mu^{(i)}(x)$
\begin{align*}
\mu^{(i)}(x)
\leq \mu^{(i)}(y)+|\mu^{(i)}(x)-\mu^{(i)}(y)|
\leq
\mu^{(i)}(y)+LV_j^{\beta} ,
\end{align*}
which upon integrating over $y$ yields $\mu^{(i)}(x)\leq \bar{\mu}^{(i)}_j+LV_j^{\beta}$. Similarly, it holds that $\mu^{(i)}(x)\geq \bar{\mu}^{(i)}_j-LV_j^{\beta}$ such that for all $x\in B_j$ we have $|\mu^{(i)}(x)-\bar{\mu}^{(i)}_j|\leq LV_j^{\beta}$. Next, note that the map $[0,1]\ni z\mapsto z^2$ is Lipschitz continuous with constant 2 which implies that $({\mu^{(i)}(x)})^2$ is $(2L, \beta)$-H\"{o}lder. This, together with the $(L, \beta)$-H\"{o}lder continuity of $(\sigma^2)^{(i)}(x)=\E({Y_t^{(i)}}^2|X_t=x)-({\mu^{(i)}(x)})^2$ implies that $\E({Y_t^{(i)}}^2|X_t=x)$ is $(3L, \beta)$-H\"{o}lder continuous. Thus, by similar arguments as above $\envert[0]{\E({Y_t^{(i)}}^2|X_t=x)-\E({Y_t^{(i)}}^2|X_t\in B_j)}\leq 3LV^{\beta}$ for all $x\in B_j$. The mean value theorem also yields that $\envert[0]{({\mu^{(i)}(x)})^2-{\bar{\mu}^{(i)}_j}^2}\leq 2LV_j^{\beta}$ for all $x\in B_j$. Therefore, 
\begin{align*}
\envert[1]{(\sigma^2)^{(i)}(x)-(\bar{\sigma}^2)^{(i)}_j}
&=
\envert[1]{\E({Y_t^{(i)}}^2|X_t=x)-({\mu^{(i)}(x)})^2-[\E({Y_t^{(i)}}^2|X_t\in B_j)-{\bar{\mu}^{(i)}_j}^2]}\\
&\leq 5LV_j^{\beta}
\end{align*}

%
%
\noindent Thus, for $x\in B_j$,
\begin{align*}
f^{(\star)}(x)
\leq
f_j^{(*)}+c_1V_j^{\beta}.
\end{align*}
A similar argument to the above yields that for all $x\in B_j$
\begin{align*}
f^{(\bar{\pi}_t)}(x)
\geq
\bar{f}_j^{(\bar{\pi}_t)}-c_1V_j^{\beta}.
\end{align*}
Next we define $\tilde{R}_j(\bar{\pi})=\sum_{t=1}^{N_{B_j}(N)}\del[2]{f_j^{(*)}-\bar{f}_j^{(\hat{\pi}_{B_j,t)})}}$. This corresponds to the regret associated with a  treatment problem without covariates where treatment $i$ yields reward $\bar{f}_j^{(i)}$, and the best treatment yields $f_j^{(*)}=\max_i \bar{f}_j^{(i)}\leq \bar{f}_j^\star$. Therefore, we can write 
\begin{align*}
R_j(\bar{\pi})
&=
\sum_{t=1}^N\left(f^{(\star)}(X_{t})-f^{(\hat{\pi}_{B_j,N_{B_j}(t)})}(X_{t})\right)1_{\left(X_{t}\in B_j\right)}
\leq
\sum_{t=1}^N\left(f_j^{(*)}-\bar{f}_j^{(\hat{\pi}_{B_j,N_{B_j}(t)})}+2c_1V_j^{\beta}\right)1_{\left(X_{t}\in B_j\right)}\\
&=
\tilde{R}_j(\bar{\pi})+2c_1V_j^{\beta}N_{B_j}(N),
\end{align*}
where $N_{B_j}(N)$ is the number of observations falling in bin $j$ given that there are $N$ observations in total. Taking expectations, and using that the density of $X_t$ is bounded from above implies that $\mathbb{E}\left[N_j(N)\right]\leq \bar{c}n\bar{B}_j$ gives
\begin{align}
\mathbb{E}\left[R_j(\bar{\pi})\right]
\leq
\mathbb{E}\left[\tilde{R}_j(\bar{\pi})\right]+n\bar{c}\bar{B}_jc_1V_j^{\beta}.\label{RegretTransformLipschitzNoMargin}
\end{align}
Since $\mathbb{E}\sbr[1]{\tilde{R}_j(\bar{\pi})}$ is the expected regret of a treatment problem without covariates we can apply Theorem \ref{LipschitzBound} with the following values 
\begin{align*}
\Delta=\sqrt{\frac{\overline{m}K\log(\overline{m}K)}{n\bar{B}_j}},\quad \gamma=\mathcal{K}L,\quad T=n\bar{B}_j,
\end{align*}
for each bin $j=1,...,F$ to obtain the following bound on the regret accumulated across any group $j$:
\begin{align*}
\mathbb{E}\left[R_j(\bar{\pi})\right]&\leq C\left[\sqrt{\overline{m}K\log(\overline{m}K)n\bar{B}_j}+n\bar{B}_jV_j^{\beta}\right].
\end{align*}
Thus, adding up the expected regret over all $F$ groups yields
\begin{align*}
\mathbb{E}\left[R_N(\bar{\pi})\right]\leq C\sum_{j=1}^F\left[\sqrt{\overline{m}K\log(\overline{m}K)n\bar{B}_j}+n\bar{B}_jV_j^{\beta}\right].
\end{align*}
\end{proof}

\begin{proof}[Proof of Corollary \ref{Cor:SimpleBins}]
The result follows from Theorem \ref{RegretBoundBinsLipschitzNoMargin} upon noting that  $\bar{B}_j=P^{-d}$ and $V_j=\sqrt{d}P^{-1}$ for $j=1,...,P$ (and ignoring the constant $\sqrt{d}$) with $P$ as in the theorem completes the proof.
\end{proof}

\begin{proof}[Proof of Theorem \ref{CovMiniMax}]
Define $\underline{S}=\mathcal{S}(\alpha, \beta, L, \mathcal{K},d, \bar{c}, \bar{m})$ to be the subset of $\mathcal{S}(\beta, L, \mathcal{K},d, \bar{c}, \bar{m})=:\bar{S}$ which also satisfies the margin condition. Then, for $\bar{m}=1, K=2$, all $\alpha>0$ and any policy $\pi$
\begin{align*}
\sup_{\bar{S}}\E\left[R_N(\pi)\right]
&\geq
\sup_{\underline{S}}\E\left[R_N(\pi)\right]
\geq
Cn^{1-\frac{\beta+\beta\alpha}{2\beta+d}}
=
Cn^{1-\frac{\beta}{2\beta+d}}\cdot n^\frac{-\beta\alpha}{2\beta+d}
\end{align*}
for a constant $C$ not depending on $\alpha$ and where the second inequality follows from Theorem 4.1 in \cite{rigollet2010nonparametric}.
Now notice that for every $n$ there exists an $\alpha$ sufficiently small such that $n^\frac{-\beta\alpha}{2\beta+d}>\frac{1}{2}$. Thus, $\sup_{\bar{S}}\E\left[R_N(\pi)\right]\geq \frac{C}{2}n^{1-\frac{\beta}{2\beta+d}}$ as desired.
\end{proof}

\begin{proof}[Proof of Theorem \ref{RegretExogenousGroups}]
The proof follows from only considering the contribution to regret coming from $\mathbb{E}[\tilde{R}_j(\bar{\pi})]$ in (\ref{RegretTransformLipschitzNoMargin}) of the proof of Theorem \ref{RegretBoundBinsLipschitzNoMargin}. 
\end{proof}

\begin{proof}[Proof of Theorem \ref{ISR}]
The proof is similar to that found in \cite{Tsybakov2004} and \cite{rigollet2010nonparametric}. Fix $\delta<\delta_0$. Then, for any policy $\pi$
\begin{align*}
R_N(\pi)&\geq\delta\sum_{t=1}^N1_{\{f^{(\star)}(X_t)-f^{(\pi_t)}(X_t)>\delta\}}\\
&\geq \delta\left(S_N(\pi)- \sum_{t=1}^N1_{\{0<|f^{(\star)}(X_t)-f^{(\pi_t)}(X_t)|\leq\delta\}}\right)\\
&\geq  \delta\left(S_N(\pi)- \sum_{t=1}^N1_{\{0<|f^{(\star)}(X_t)-f^{(\sharp)}(X_t)|\leq\delta\}}\right)
\end{align*}
Since $S_n(\pi)\leq N$ there exists a $c>0$ not depending on $N$ such that $\left(\frac{S_n(\pi)}{cn}\right)^{\frac{1}{\alpha}}<\delta_0$. Thus, we can set $\delta=\left(\frac{S_n(\pi)}{cn}\right)^{\frac{1}{\alpha}}$ and use the margin condition upon integration on both sides of the above display to get (\ref{ISR1}). To obtain (\ref{ISR2}) insert (\ref{part2}) into (\ref{ISR1}). 
\end{proof}

\begin{proof}
The idea of the proof is based on localization as in the proof of Theorem \ref{RegretBoundBinsLipschitzNoMargin}. As in that proof let $c_1=6L\mathcal{K}+1$. Note that $Nr_N=N\sum_{j=1}^Fr_j$ where
\begin{align*}
r_j=\del[1]{f^{\star}(X_t)-f^{I_{n+1}}(X_t)}1_{\cbr[0]{X_t\in B_j}}
\end{align*} 
Following the arguments in Theorem \ref{RegretBoundBinsLipschitzNoMargin} gives $f^{\star}(x)\leq f_j^{(*)}+c_1V_j^\beta$ and $f^{I_{n+1}}(x)\geq \bar{f}_j^{I_{n+1}}-c_1V_j^\beta$ for all $x\in B_j$. Thus,
\begin{align*}
r_j
\leq 
\del[1]{f_j^{(*)}-\bar{f}_j^{I_{n+1}}+2c_1V_j^\beta}1_{\cbr[0]{X_t\in B_j}} 
\end{align*}
which upon using that $\mathbb{P}(B_j)\leq \bar{c}\bar{B}_j$ yields that
\end{proof}

\begin{proof}[Proof of Theorem \ref{Ethic1}]
The proof is identical to the proof of Theorem \ref{Thm3.2} but with with $n=\mathbb{E}(N)$ replaced by $\bar{c}n\bar{B}_j\geq \mathbb{E}\del[1]{N_{B_j}(N)}$. Thus, the expected number of assignments is replaced by an upper bound on the expected number of individuals falling in group $j$ and the result of Theorem \ref{Thm3.2} is applied on each group separately.
\end{proof}

\begin{proof}[Proof of Theorem \ref{RegretBound_DiscCov}] 
The proof is similar to the proof of Theorem \ref{RegretBoundBinsLipschitzNoMargin} once we fix a value of the discrete covariates. Let $c$ denote a positive constant which may change from line to line. By the construction of the treatment policy it follows that the regret can be written as $R_N(\tilde{\pi})=\sum_{a\in A}\sum_{j=1}^{F_a}R_{a,j}(\hat{\pi})$, where
\begin{align*}
R_{a,j}(\hat{\pi})=\sum_{t=1}^N\left(f^{(\star)}(X_{t})-f^{(\hat{\pi}_{(\cbr[0]{a}\times B_{a,j}), N_{{a,j}}(t)})}(X_{t})\right)1_{\left(X_{t,D}=a,X_{t,C}\in B_{a,j}\right)}.
\end{align*}
For any bin $B_{a,j}$ define
\begin{align*}
\bar{\mu}^{(i)}_{a,j}=\E(Y_{t}^{(i)}|X_{t,D}=a, X_{t,C}\in B_{a,j})=\frac{1}{\mathbb{P}_X(X_{t,D}=a,X_{t,C}\in B_{a,j})}\int_{a\times B_{a,j}}\mu^{(i)}(x)d\mathbb{P}_X(x)
\end{align*}
and
\begin{align*}
(\bar{\sigma}^2)^{(i)}_{a,j}
&=
Var(Y_t^{(i)}|X_{t,D}=a,X_{t,C}\in B_{a,j})\\
&=
\E({Y_t^{(i)}}^2|X_{t,D}=a,X_{t,C}\in B_{a,j})-[\E(Y_t^{(i)}|X_{t,D}=a,X_{t,C}\in B_{a,j})]^2
\end{align*}
Furthermore, let $\bar{f}_{a,j}^{(i)}=f(\bar{\mu}^{(i)}_{a,j},(\bar{\sigma}^2)^{(i)}_{a,j})$ with $f_{a,j}^{(*)}=\max_{1\leq i\leq K+1}\bar{f}_{a,j}^{(i)}$. By exactly the same arguments as in the proof of Theorem \ref{RegretBoundBinsLipschitzNoMargin} we now get that for $x\in \cbr[0]{a}\times B_{a,j}$,
\begin{align*}
f^{(\star)}(x)
\leq
f_{a,j}^{(*)}+cV_{a,j}^{\beta}.
\end{align*}
as well as,
\begin{align*}
f^{(\hat{\pi}_{(a\times B_{a,j}), N_{a,j}(t)})}(x)
\geq
\bar{f}_{a,j}^{(\hat{\pi}_{(\cbr[0]{a}\times B_{a,j}), N_{a,j}(t)})}-cV_{a,j}^{\beta}.
\end{align*}
Next we define $\tilde{R}_{a,j}(\tilde{\pi})=\sum_{t=1}^{N_{a,j}}\del[2]{f_{a,j}^{(*)}-\bar{f}_{a,j}^{(\hat{\pi}_{(\cbr[0]{a}\times B_{a,j}),t})}}$. This corresponds to the regret associated with a treatment problem without covariates where treatment $i$ yields reward $\bar{f}_{a,j}^{(i)}$, and the best treatment yields $f_{a,j}^{(*)}=\max_i \bar{f}_{a,j}^{(i)}\leq \bar{f}_{a,j}^\star$. Therefore, we can write 
\begin{align*}
R_{a,j}(\hat{\pi})
&=
\sum_{t=1}^N\left(f^{(\star)}(X_{t})-f^{(\hat{\pi}_{(a\times B_{a,j}), N_{a,j}(t)})}(X_{t})\right)1_{\left(X_{t,D}=a,X_{t,C}\in B_{a,j}\right)}\\
&\leq
\sum_{t=1}^N\left(f_{a,j}^{(*)}-\bar{f}_{a,j}^{(\hat{\pi}_{(\cbr[0]{a}\times B_{a,j}), N_{a,j}(t)})}+2cV_{a,j}^{\beta}\right)1_{\left(X_{t,D}=a,X_{t,C}\in B_{a,j}\right)}\\
&=
\tilde{R}_{a,j}(\hat{\pi})+2cV_{a,j}^{\beta}N_{a,j}(N),
\end{align*}
where $N_{a,j}(N)$ is the number of observations for which $x\in a\times B_{a,j}$ given that there are $N$ observations in total. Taking expectations, and using that $N$ is independent of all other random variables implies $\mathbb{E}\left[N_{a,j}(N)\right]\leq n\mathbb{P}(X_{t,D}=a,X_{t,C}\in B_{a,j})$ gives
\begin{align*}
\mathbb{E}\left[R_{a,j}(\hat{\pi})\right]
\leq
\mathbb{E}\left[\tilde{R}_{a,j}(\hat{\pi})\right]+n\mathbb{P}(X_{t,D}=a,X_{t,C}\in B_{a,j})cV_{a,j}^{\beta}.
\end{align*}
Since $\mathbb{E}\sbr[1]{\tilde{R}_{a,j}(\hat{\pi})}$ is the expected regret of a treatment problem without covariates we can apply Theorem \ref{LipschitzBound} with the following values 
\begin{align*}
\Delta=\sqrt{\frac{\overline{m}K\log(\overline{m}K)}{n\mathbb{P}(X_{t,D}=a,X_{t,C}\in B_{a,j})}},\quad \gamma=\mathcal{K}L,\quad T=n\mathbb{P}(X_{t,D}=a,X_{t,C}\in B_{a,j}),
\end{align*}
for each $a\in A$ and $B_{a,j},\ j=1,...,F_{a}$ to obtain the following bound on the regret accumulated across any group:
\begin{align*}
\mathbb{E}\left[R_{a,j}(\bar{\pi})\right]&\leq c\left[\sqrt{\overline{m}K\log(\overline{m}K)n\mathbb{P}(X_{t,D}=a,X_{t,C}\in B_{a,j})}+n\mathbb{P}(X_{t,D}=a,X_{t,C}\in B_{a,j})_jV_j^{\beta}\right].
\end{align*}
Thus, adding up the expected regret over all groups yields
\begin{align*}
\mathbb{E}\left[R_N(\tilde{\pi})\right]\leq c\sum_{a\in A}\sum_{j=1}^{F_a}\left[\sqrt{\overline{m}K\log(\overline{m}K)n\mathbb{P}(X_{t,D}=a,X_{t,C}\in B_{a,j})}+n\mathbb{P}(X_{t,D}=a,X_{t,C}\in B_{a,j})V_{a,j}^{\beta}\right].
\end{align*}
\end{proof}

\subsection{Proof of Theorems in Section \ref{Sec:Delay}}
\begin{proof}[Proof of Theorem \ref{LipschitzBoundDelay}]
Define $\epsilon_{s}=u(s,T)=16\sqrt{\frac{2\bar{a}^2}{s}\overline{\log}\left(\frac{T}{s}\right)}$. In the following we will distinguish between two types of batches, namely batches of individuals that have to be assigned a treatment, and batches of information on the outcome of previously assigned treatments. The latter type of batches will be the key object of interest when determining whether or not to eliminate a given treatment, whereas the former will be relevant when counting the total regret from running the treatment policy. In this proof we let $\underline{B}(s)$ denote the minimal number of observed outcomes per treatment based on $s$ batches of information. Consider a batch $b$ of information. Recall that if the optimal treatment as well as some treatment $i$ have not been eliminated, then the optimal treatment will eliminate treatment $i$ if $\hat{\Delta}_i(\underline{B}(b))\geq\gamma\epsilon_{\underline{B}(b)}$, and treatment $i$ will eliminate the optimal treatment if $\hat{\Delta}_i(\underline{B}(b))\leq-\gamma\epsilon_{\underline{B}(b)}$.

To be able to say something about when either of these two events occurs we introduce the (unknown) quantity, $\tau_i^*$, which is defined through the relation
\begin{align*}
\Delta_i=24\gamma\sqrt{\frac{2\bar{a}^2}{\tau_i^*}\overline{\log}\left(\frac{T}{\tau_i^*}\right)}, \qquad i=1,...,K.
\end{align*}
Since $\tau_i^*$ in general will not be an integer, we also define $\tau_i=\ceil{\tau_i^*}$. Next introduce the hypothetical batch (of information) $b_i=\min\{l:\underline{B}(l)\geq \tau_i^*\}$. It is the first batch of information after which we have more than $\tau_i^*$ observations of the outcome of treatment $i$. Notice that
\begin{align}
\tau_i^*&\leq \underline{B}(b_i)\leq C\left(\frac{\bar{a}^2\gamma^2}{\Delta_i^2}\overline{\log}\left(\frac{T\Delta_i^2}{1152\bar{a}^2\gamma^2}\right)+\overline{m}\right),\label{delay.1}\\
\tau_i&\leq \underline{B}(b_i),\label{delay.2}\\
\underline{B}(b_i)&\leq \tau_i+\overline{m},\label{delay.3}
\end{align}
Notice that $1\leq \tau_1\leq ...\leq \tau_K$ and $1\leq b_1 \leq ...\leq b_K$.
Define the following events:
\begin{align*}
\mathcal{A}_i&=\{\text{The optimal treatment has not been eliminated after batch }b_i\text{ has been observed} \},\\
\mathcal{B}_i&=\{\text{Every treatment }j\in\{1,...,i\} \text{ has been eliminated after batch } b_j\text{ has been observed}\}.\\
\end{align*}
Furthermore, let $\mathcal{C}_i=\mathcal{A}_i\cap \mathcal{B}_i$, and observe that $\mathcal{C}_1\supseteq...\supseteq\mathcal{C}_K$. For any $i=1,...,K$, the contribution to regret incurred after batch $b_i$ of information is at most $\Delta_{i+1}N$ on $\mathcal{C}_i$. In what follows we fix a treatment, $K_0$, which we will be specific about later. Using this and letting $m$ denote the expected number of observations in a batch we get the following decomposition of expected regret:
\begin{align}
\mathbb{E}\left[R_N\left(\hat{\pi}\right)\right]&=\mathbb{E}\left[R_N\left(\hat{\pi}\right)\left(\sum_{i=1}^{K_0}1_{\mathcal{C}_{i-1}\backslash\mathcal{C}_i}+1_{\mathcal{C}_{K_0}}\right)\right]\nonumber\\
&\leq n\sum_{i=1}^{K_0}\Delta_i\mathbb{P}\left(\mathcal{C}_{i-1}\backslash\mathcal{C}_i\right)+\sum_{i=1}^{K_0}B_i(b_i)\Delta_i+n\Delta_{K_0+1}+Dm,\label{ExpRegretDelay}
\end{align}
where the last term is due to the fact that the delay means that the all treatment allocations during the first $D+1$ batches have to be made without any information about the treatment outcomes. For every $i=1,...,K$ the event $\mathcal{C}_{i-1}\backslash\mathcal{C}_i$ can be decomposed as follows
\begin{align*}
\mathcal{C}_{i-1}\backslash\mathcal{C}_i=\left(\mathcal{C}_{i-1}\cap \mathcal{A}_i^c\right)\cup\left(\mathcal{B}_i^c\cap\mathcal{A}_i\cap \mathcal{B}_{i-1}\right).
\end{align*}
Therefore, the first term on the right-hand side of (\ref{ExpRegretDelay}) can be written as
\begin{align}
n\sum_{i=1}^{K_0}\Delta_i\mathbb{P}\left(\mathcal{C}_{i-1}\backslash\mathcal{C}_i\right)=n\sum_{i=1}^{K_0}\Delta_i\mathbb{P}\left(\mathcal{C}_{i-1}\cap\mathcal{A}_i^c\right)+n\sum_{i=1}^{K_0}\Delta_i\mathbb{P}\left(\mathcal{B}_i^c\cap\mathcal{A}_i\cap \mathcal{B}_{i-1}\right).\label{FirstTermDelay}
\end{align}
Notice that the first term on the right-hand side will be zero if $b_{i-1}=b_i$. On the event $\mathcal{B}_i^c\cap\mathcal{A}_i\cap \mathcal{B}_{i-1}$ the optimal treatment has not eliminated treatment $i$ at batch $b_i$. Therefore, for the last term on the right hand side of equation (\ref{FirstTermDelay}) we find that
\begin{align}
\mathbb{P}\left(\mathcal{B}_i^c\cap\mathcal{A}_i\cap \mathcal{B}_{i-1}\right)&\leq \mathbb{P}\left(\hat{\Delta}_i(\underline{B}(b_i))\leq \gamma\epsilon_{\underline{B}(b_i)}\right)\nonumber\\
&\leq \mathbb{P}\left(\hat{\Delta}_i(\underline{B}(b_i))-\Delta_i\leq \gamma\epsilon_{\tau_i}-\Delta_i\right)\nonumber\\
&\leq
 \mathbb{E}\left[\mathbb{P}\left(|\hat{\Delta}_i(\underline{B}(b_i))-\Delta_i|\geq \frac{1}{2}\gamma\epsilon_{\tau_i}|\underline{B}(b_i)\right)\right].\label{12.1}
\end{align}
For any $s\geq \tau_i$ we have that
\begin{align}
&\mathbb{P}\left(|\hat{\Delta}_i(s)-\Delta_i|\geq \frac{1}{2}\gamma\epsilon_{\tau_i}\right)\notag\\
&\leq \mathbb{P}\left(|f(\hat{\mu}_s^{(*)},(\hat{\sigma}^2_s)^{(*)})-f(\hat{\mu}^{(i)}_s,(\hat{\sigma}^2_s)^{(i)})+f(\mu^{(i)},(\sigma^2)^{(i)})-f(\mu^{(*)},(\sigma^2)^{(*)})|\geq \frac{1}{2}\gamma\epsilon_{\tau_i}\right)\nonumber\\
&\leq \mathbb{P}\left(|f(\hat{\mu}_s^{(*)},(\hat{\sigma}^2_s)^{(*)})-f(\mu^{(*)},(\sigma^2)^{(*)})|\geq \frac{1}{4}\gamma\epsilon_{\tau_i}\right)
+\mathbb{P}\left(|f(\hat{\mu}_s^{(i)},(\hat{\sigma}^2_s)^{(i)})-f(\mu^{(i)},(\sigma^2)^{(i)})|\geq \frac{1}{4}\gamma\epsilon_{\tau_i}\right).\label{12.2}
\end{align}
Furthermore, for any $j\in\{i,*\}$, we have
\begin{align}
&\mathbb{P}\left(|f(\hat{\mu}_s^{(j)},(\hat{\sigma}^2_s)^{(j)})-f(\mu^{(j)},(\sigma^2)^{(j)})|\geq \frac{1}{4}\gamma\epsilon_{\tau_i}\right)\notag\\
&\leq 
\mathbb{P}\left(|\hat{\mu}_s^{(j)}-\mu^{(j)}|+|(\hat{\sigma}^2_s)^{(j)}-(\sigma^2)^{(j)}|\geq \frac{1}{4\mathcal{K}}\gamma\epsilon_{\tau_i}\right)\nonumber\\
&\leq \mathbb{P}\left(|\hat{\mu}_s^{(j)}-\mu^{(j)}|\geq \frac{1}{8\mathcal{K}}\gamma\epsilon_{\tau_i}\right)+\mathbb{P}\left(|(\hat{\sigma}^2_s)^{(j)}-(\sigma^2)^{(j)}|\geq \frac{1}{8\mathcal{K}}\gamma\epsilon_{\tau_i}\right).\label{12.3}
\end{align}
By the mean value theorem we have that
\begin{align}
\mathbb{P}\left(|(\hat{\sigma}^2_s)^{(j)}-(\sigma^2)^{(j)}|\geq \frac{1}{8\mathcal{K}}\gamma\epsilon_{\tau_i}\right)\leq \mathbb{P}\left(|\hat{\mu}_s^{(j)}-\mu^{(j)}|\geq \frac{1}{32\mathcal{K}}\gamma\epsilon_{\tau_i}\right)+\mathbb{P}\left(|(\hat{\mu}_{2,s})^{(j)}-\mu_2^{(j)}|\geq \frac{1}{16\mathcal{K}}\gamma\epsilon_{\tau_i}\right),\label{12.4}
\end{align}
where $\mu_2=\mathbb{E}\left[Y_1^2\right]$ and $\hat{\mu}_{2,s}=\frac{1}{s}\sum_{i=1}^sY_i^2$. Combining equations (\ref{12.1}), (\ref{12.2}), (\ref{12.3}) and (\ref{12.4}) and using Hoeffding's inequality to each of the three terms as well as the fact that $\gamma\geq\mathcal{K}$ we arrive at the following bound,
\begin{align*}
\mathbb{P}\left(|\hat{\Delta}_i(s)-\Delta_i|\geq \gamma\epsilon_{\tau_i}\right)&\leq C\exp\left(-\frac{1}{1024\bar{a}}\epsilon_{\tau_i}^2 s\right)\\
&\leq C\exp\left(-\frac{1}{1024\bar{a}}\epsilon_{\tau_i}^2 \tau_i\right)\\
&=C\exp\left(-\overline{\log}\left(\frac{T}{\tau_i}\right)\right)\\
&\leq C\frac{\tau_i}{T}.
\end{align*}
Thus,
\begin{align*}
\mathbb{P}\left(\mathcal{B}_i^c\cap\mathcal{A}_i\cap \mathcal{B}_{i-1}\right)&\leq C\frac{\tau_i}{T}
\end{align*} 
On the event $\mathcal{C}_{i-1}\cap\mathcal{A}_i^c$ the optimal treatment is eliminated between the time batch $b_{i-1}+1$ and $b_i$ of information arrives. Furthermore, every suboptimal treatment $j\leq i-1$ has also been eliminated. Therefore, the probability of this event can be bounded as follows:
\begin{align*}
\mathbb{P}\left(\mathcal{C}_{i-1}\cap\mathcal{A}_i^c\right)&\leq\mathbb{P}\left(\exists(j,s),i\leq j\leq K,b_{i-1}+1\leq s\leq b_i;\hat{\Delta}_j(\underline{B}(s))\leq -\gamma \epsilon_{\underline{B}(s)}\right)\\
&\leq \sum_{j=i}^K\mathbb{P}\left(\exists s,b_{i-1}+1\leq s\leq b_i;\hat{\Delta}_j(\underline{B}(s))\leq -\gamma \epsilon_{\underline{B}(s)}\right)\\
&= \sum_{j=i}^K\left[\Phi_j(b_i)-\Phi_j(b_{i-1})\right],
\end{align*}
where $\Phi_j(b)=\mathbb{P}\left(\exists s\leq b; \hat{\Delta}_j(\underline{B}(s))\leq -\gamma \epsilon_{\underline{B}(s)}\right)$. We proceed to bounding terms of the form $\Phi_j(b_i)$ for $j\geq i$. 
\begin{align*}
\mathbb{P}\left(\exists s\leq b_i; \hat{\Delta}_j(\underline{B}(s))\leq -\gamma \epsilon_{\underline{B}(s)}\right)&\leq \mathbb{P}\left(\exists s\leq b_i; \hat{\Delta}_j(\underline{B}(s))-\Delta_j\leq -\gamma \epsilon_{\underline{B}(s)}\right)\\
&\leq\mathbb{P}\left(\exists s\leq B_j(b_i);\hat{\Delta}_j(s)-\Delta_j\leq -\gamma \epsilon_{s}\right)\\
&\leq\mathbb{P}\left(\exists s\leq \tau_i+\overline{m};\hat{\Delta}_j(s)-\Delta_j\leq -\gamma \epsilon_{s}\right)\\
&\leq \mathbb{P}\left(\exists s\leq \tau_i+\overline{m};|f(\hat{\mu}_s^{(j)},(\hat{\sigma}_s^2)^{(j)})-f(\mu^{(j)},(\sigma^2)^{(j)})|\geq \frac{1}{2}\gamma \epsilon_{s}\right)\\
&+ \mathbb{P}\left(\exists s\leq \tau_i+\overline{m};|f(\hat{\mu}_s^{(*)},(\hat{\sigma}_s^2)^{(*)})-f(\mu^{(*)},(\sigma^2)^{(*)})|\geq \frac{1}{2}\gamma \epsilon_{s}\right).
\end{align*}
For any $j\in\{i,...,K,*\}$ we find that
\begin{align*}
&\mathbb{P}\left(\exists s\leq \tau_i+\overline{m};|f(\hat{\mu}_s^{(j)},(\hat{\sigma}_s^2)^{(j)})-f(\mu^{(j)},(\sigma^2)^{(j)})|\geq \frac{1}{2}\gamma \epsilon_{s}\right)\notag\\
&\leq
 \mathbb{P}\left(\exists s\leq \tau_i+\overline{m};|\hat{\mu}_s^{(j)}-\mu^{(j)}|\geq \frac{1}{4\mathcal{K}}\gamma \epsilon_{s}\right)
+\mathbb{P}\left(\exists s\leq \tau_i+\overline{m};|(\hat{\sigma}_s^2)^{(j)})-(\sigma^2)^{(j)})|\geq \frac{1}{4\mathcal{K}}\gamma \epsilon_{s}\right)\\
&\leq \mathbb{P}\left(\exists s\leq \tau_i+\overline{m};|\hat{\mu}_s^{(j)}-\mu^{(j)}|\geq \frac{1}{4\mathcal{K}}\gamma \epsilon_{s}\right)
+\mathbb{P}\left(\exists s\leq \tau_i+\overline{m};|(\hat{\mu}_{2,s})^{(j)}-\mu_2^{(j)}|\geq \frac{1}{8\mathcal{K}}\gamma \epsilon_{s}\right)\\
&+\mathbb{P}\left(\exists s\leq \tau_i+\overline{m};|\hat{\mu}_s^{(j)}-\mu^{(j)}|\geq \frac{1}{16\mathcal{K}}\gamma \epsilon_{s}\right)\\
&\leq C\frac{\tau_i+\overline{m}}{T},
\end{align*}
where we once more have used equation (\ref{delay.3}) and Lemma A.1 in \cite{CovBandit}. Using this we find that
\begin{align}
\sum_{i=1}^{K_0}\Delta_i\mathbb{P}\left(\mathcal{C}_{i-1}\cap\mathcal{A}_i^c\right)&\leq \sum_{i=1}^{K_0}\Delta_i\sum_{j=i}^{K}\left[\Phi_j(b_i)-\Phi_j(b_{i-1})\right]\nonumber\\
&\leq \sum_{j=1}^{K}\sum_{i=1}^{j\wedge K_0-1}\Phi_j(b_i)\left(\Delta_i-\Delta_{i+1}\right)+\sum_{j=K_0}^K\Delta_{K_0}\Phi_j(b_{K_0})+\sum_{j=1}^{K_0-1}\Delta_j\Phi_j(b_j)\nonumber\\
&\leq C\left(\frac{1}{T}\sum_{j=1}^{K}\sum_{i=1}^{j\wedge K_0-1}\left(\tau_i+\overline{m}\right)\left(\Delta_i-\Delta_{i+1}\right)+\frac{1}{T}\sum_{j=1}^K\Delta_{j\wedge K_0}\left(\tau_{j\wedge K_0}+\overline{m}\right)\right).\label{12.5}
\end{align} 
Observe that, by (\ref{delay.2}),
\begin{align}
\sum_{j=1}^{K}\sum_{i=1}^{j\wedge K_0-1}\tau_i\left(\Delta_i-\Delta_{i+1}\right)&\leq C\gamma^2\bar{a}^2\sum_{j=1}^{K}\sum_{i=1}^{j\wedge K_0-1}\frac{\left(\Delta_i-\Delta_{i+1}\right)}{\Delta_i^2}\overline{\log}\left(\frac{T\Delta_i^2}{1152\bar{a}^2\gamma^2}\right)\nonumber\\
&\leq C\bar{a}^2\gamma^2\sum_{j=1}^{K}\int_{\Delta_{j\wedge K_0}}^{\Delta_1}\frac{1}{x^2}\overline{\log}\left(\frac{Tx^2}{1152\bar{a}^2\gamma^2}\right)dx\nonumber\\
&\leq C\bar{a}^2\gamma^2\sum_{j=1}^{K}\frac{1}{\Delta_{j\wedge K_0}}\overline{\log}\left(\frac{T\Delta_{j\wedge K_0}^2}{1152\bar{a}^2\gamma^2}\right).\label{12.6}
\end{align}
The parts involving $\overline{m}$ in equation (\ref{12.5}) can be bounded by
\begin{align}
\overline{m}\sum_{j=1}^{K}\sum_{i=1}^{j\wedge K_0-1}\left(\Delta_i-\Delta_{i+1}\right)+\sum_{j=1}^K\Delta_{j\wedge K_0}\overline{m}&\leq \overline{m}K.  \label{12.7}
\end{align}
Bringing together equations (\ref{12.5}), (\ref{12.6}) and (\ref{12.7}) we see that 
\begin{align}
\sum_{i=1}^{K_0}\Delta_i\mathbb{P}\left(\mathcal{C}_{i-1}\cap\mathcal{A}_i^c\right)&\leq C\left(\frac{\bar{a}^2\gamma^2}{T}\sum_{j=1}^K\frac{1}{\Delta_{j\wedge K_0}}\overline{\log}\left(\frac{T\Delta_{j\wedge K_0}^2}{1152\bar{a}^2\gamma^2}\right)+\frac{K\overline{m}}{T}\right).
\end{align}
Combining this with equation (\ref{FirstTermDelay}) and (\ref{ExpRegretDelay}) we obtain
\begin{align}
\mathbb{E}\left[R_N\left(\hat{\pi}\right)\right]&\leq C\biggl(\frac{\bar{a}^2\gamma^2n}{T}\sum_{j=1}^K\frac{1}{\Delta_{j\wedge K_0}}\overline{\log}\left(\frac{T\Delta_{j\wedge K_0}^2}{1152\bar{a}^2\gamma^2}\right)
+\frac{n\bar{a}^2\gamma^2}{T}\sum_{j=1}^{K_0}\frac{1}{\Delta_{j}}\overline{\log}\left(\frac{T\Delta_{j}^2}{1152\bar{a}^2\gamma^2}\right)\nonumber\\
&+\sum_{i=1}^{K_0}B_i(b_i)\Delta_i+n\Delta_{K_0+1}+\frac{n\overline{m}K}{T}+mD\biggr)\nonumber\\
&\leq C\biggl(\left(1+\frac{n}{T}\right) \bar{a}^2\gamma^2\sum_{j=1}^{K_0}\frac{1}{\Delta_{j}}\overline{\log}\left(\frac{T\Delta_{j}^2}{1152\bar{a}^2\gamma^2}\right)+\frac{\bar{a}^2\gamma^2n}{T}\frac{K-K_0}{\Delta_{K_0}}\overline{\log}\left(\frac{T\Delta_{K_0}^2}{1152\bar{a}^2\gamma^2}\right)\nonumber\\
&+n\Delta_{K_0+1}+\frac{n\overline{m}K}{T}+\overline{m}D\biggr).
\end{align}
Fix $\Delta>0$ and let $K_0$ be such that $\Delta_{K_0+1}=\Delta^-$. Define the function $\phi(\cdot)$ by
\begin{align*}
\phi(x)=\frac{1}{x}\overline{\log}\left(\frac{Tx^2}{1152\bar{a}^2\gamma^2}\right),
\end{align*}
and notice that $\phi(x)\leq 2e^{-1/2}\phi(x')$ for any $x\geq x'\geq 0$. Using this with $x'=\Delta$ and $x=\Delta_i$ for $i\leq K_0$ we obtain the following bound on the expected regret.
\begin{align}
\mathbb{E}\left[R_N(\hat{\pi})\right]\leq C\left(\frac{\bar{a}^2\gamma^2K}{\Delta}\left(1+\frac{n}{T}\right)\overline{\log}\left(\frac{T\Delta^2}{1152\bar{a}^2\gamma^2}\right)+n\Delta^{-}+\frac{n\overline{m}K}{T}+mD\right).
\end{align}
Note that we by definition we have that $m\leq\overline{m}$. The theorem then follows by arguments similar to those in the proof of theorem \ref{NoCovRegret}. 

\qed
\end{proof}

\begin{proof}[Proof of Theorem \ref{Delay}]
Recall equation (\ref{RegretTransformLipschitzNoMargin}). Applying Theorem \ref{LipschitzBoundDelay} with the following values 
\begin{align*}
\Delta=\sqrt{\frac{\overline{m}K\bar{a}\log(\overline{m}K/\bar{a})}{n\bar{B}_j}},\quad \gamma=\mathcal{K}L,\quad T=n\bar{B}_j,
\end{align*}
for each bin $j=1,...,F$, we obtain the following bound on the regret accumulated across the any bin $j$:
\begin{align*}
\mathbb{E}\left[R_j(\bar{\pi})\right]&\leq c\left[\sqrt{\overline{m}\bar{a}^3K\log(\overline{m}K/\bar{a})n\bar{B}_j}+n\bar{B}_jV_j^{\beta}+\overline{m}K+m_jD\right],
\end{align*}
where $m_j$ is the expected batch size associated with bin $j$. Note that $m_j\leq \bar{c}\overline{m}\bar{B}_j$. Thus,
\begin{align*}
\mathbb{E}\left[R_N(\bar{\pi})\right]\leq c\left(\sum_{j=1}^F\left[\sqrt{\overline{m}K\bar{a}^3\log(\overline{m}K/\bar{a})n\bar{B}_j}+n\bar{B}_jV_j^{\beta}+\overline{m}K\right]+\overline{m}D\right).
\end{align*}
\end{proof}

\bibliographystyle{plainnat}
\bibliography{bandit}
\end{document}